\documentclass[journal]{IEEEtran}

\usepackage{graphicx}
\usepackage{array}
\usepackage{url}
\usepackage{stfloats}
\usepackage[tight]{subfigure}
\usepackage{amsmath,bm}
\usepackage{textcomp, gensymb}
\usepackage{amssymb}
\usepackage{amsfonts}
\usepackage{amsopn}
\usepackage{psfrag}
\usepackage{enumitem}
\usepackage{amsthm}
\usepackage{accents}
\usepackage{color}
\usepackage{algorithmicx}
\usepackage{algorithm}
\usepackage{algpseudocode}
\usepackage{acronym}
\usepackage[english]{babel}
\usepackage[utf8]{inputenc}
\usepackage{booktabs}
\usepackage{multirow}
\usepackage{mdwlist}
\usepackage{epstopdf}
\usepackage{hyperref}
\usepackage{newtxtext}
\acrodef{COM}[\textsc{COM}]{Center of Mass}
\acrodef{JPS}[\texttt{JPS}]{Jump Point Search}
\acrodef{SDP}[\textsc{SDP}]{Semidefinite Program}
\acrodef{RRT}[\texttt{RRT}]{Rapidly-exploring Random Tree}
\acrodef{CBF}[\textsc{CBF}]{Control Barrier Function}
\acrodef{MPC}[\textsc{MPC}]{Model Predictive Control}
\acrodef{QP}[\textsc{QP}]{Quadratic Program}
\acrodef{SC}[\textsc{SC}]{Safe Corridor}
\acrodef{SWC}[\textsc{SWC}]{\emph{Safe Walking Corridors}}
\acrodef{RISP}[\texttt{RISP}]{\emph{Randomized Iterative Space Partitioning}}
\acrodef{IRIS}[\texttt{IRIS}]{Iterative Regional Inflation by Semidefinite Programming}
\acrodef{UAV}[\textsc{UAV}]{Unmanned Aerial Vehicle}
\acrodef{FIRI}[\texttt{FIRI}]{Fast Iterative Region Inflation}
\acrodef{CIRI}[\texttt{CIRI}]{Configuration-Space Iterative Regional Inflation}
\acrodef{AGV}[\textsc{AGV}]{Autonomous Ground Vehicles}

\newtheorem{rem}{\bf Remark}
\newtheorem{assump}{\bf Assumption}
\newtheorem{theorem}{\bf Theorem}
\newtheorem{lemma}{\bf Lemma}

\begin{document}

\title{\LARGE \bf \texttt{PathCover}: A Fast Convex Decomposition along a Path via Randomized Iterative Space Partitioning (\texttt{RISP}) on Point Clouds}

\author{Kunal S. Narkhede$^{1}$, Abhijeet M. Kulkarni$^{1}$, Guoquan Huang$^{1}$ and Ioannis Poulakakis$^{2}$
\thanks{$^1$K. S. Narkhede, A. M. Kulkarni and G. Huang are with the Department of Mechanical Engineering, University of Delaware, Newark, DE 19716, USA {\tt\small \{kunalnk, amkulk, ghuang\}@udel.edu}}%
\thanks{$^{2}$I. Poulakakis is with Robotics Institute, Athena Research Center, Marousi, Greece; School of Mechanical Engineering, National Technical University of Athens, Greece; HERON–Center
of Excellence in Robotics, Athens, Greece {\tt\small poulakas@mail.ntua.gr,
i.poulakakis@athenarc.gr}}
}

\maketitle

\begin{abstract}
Autonomous robot navigation requires rapid construction of obstacle-points-separated convex regions for trajectory planning. When obstacles are represented as point clouds from LiDAR or depth cameras, these regions must be constructed directly from finite obstacle samples while providing suitable constraints for downstream optimization. However, existing corridor-generation methods often struggle to meet real-time, sensor-rate computational requirements. To address this bottleneck, we introduce \ac{RISP}, a randomized algorithm that constructs convex polytopes from finite point-cloud data, and \texttt{PathCover}, which chains these polytopes along a reference path to form an overlapping corridor. We prove finite termination, sequential intersection, and complete path coverage with respect to the supplied finite point set. Under a probabilistic elimination condition, the sampling-and-elimination stage of \ac{RISP} has expected $O(n)$ time and unconditional worst-case $O(n^2)$ time. Extensive benchmarks on synthetic and real-world LiDAR datasets demonstrate an order-of-magnitude speedup over state-of-the-art methods in both corridor generation as well as trajectory optimization, while closed-loop quadrotor simulation and a physical quadruped traversal demonstrate integration with downstream motion planners. These results establish corridor separation from the supplied point-cloud representation and practical implementation feasibility. Source code of the entire pipeline is available at \texttt{\hyperlink{PathCover}{https://github.com/kunalnk123690/PathCover}}.
\end{abstract}

\section{Introduction}
\label{sec:intro}

\IEEEPARstart{A}{utonomous} navigation through cluttered, unstructured environments is a fundamental requirement for autonomous robot navigation. The same need recurs in autonomous urban driving~\cite{How2009TCST}, legged robot locomotion~\cite{Havoutis2020RAL}, aerial manipulation~\cite{Kim2018RAL}, \ac{AGV} operation in logistics~\cite{Gao2025Universal}, and multi-agent quadrotor coordination~\cite{Park2023TRO, Mora2023RAL}. Despite their different dynamics and operating envelopes, these systems must convert live sensor data into geometric constraints within the available planning budget. When this computation exceeds the budget, a navigation system may need to decelerate, hover, or invoke another fallback policy~\cite{How2022FASTER}. For real world robot deployment, therefore, free-space generation is not a pre-processing step to be optimized in isolation; it determines whether continuous perception translates into safe, uninterrupted motion.

\begin{figure}[t!]
\centering
{\includegraphics[width=\columnwidth]{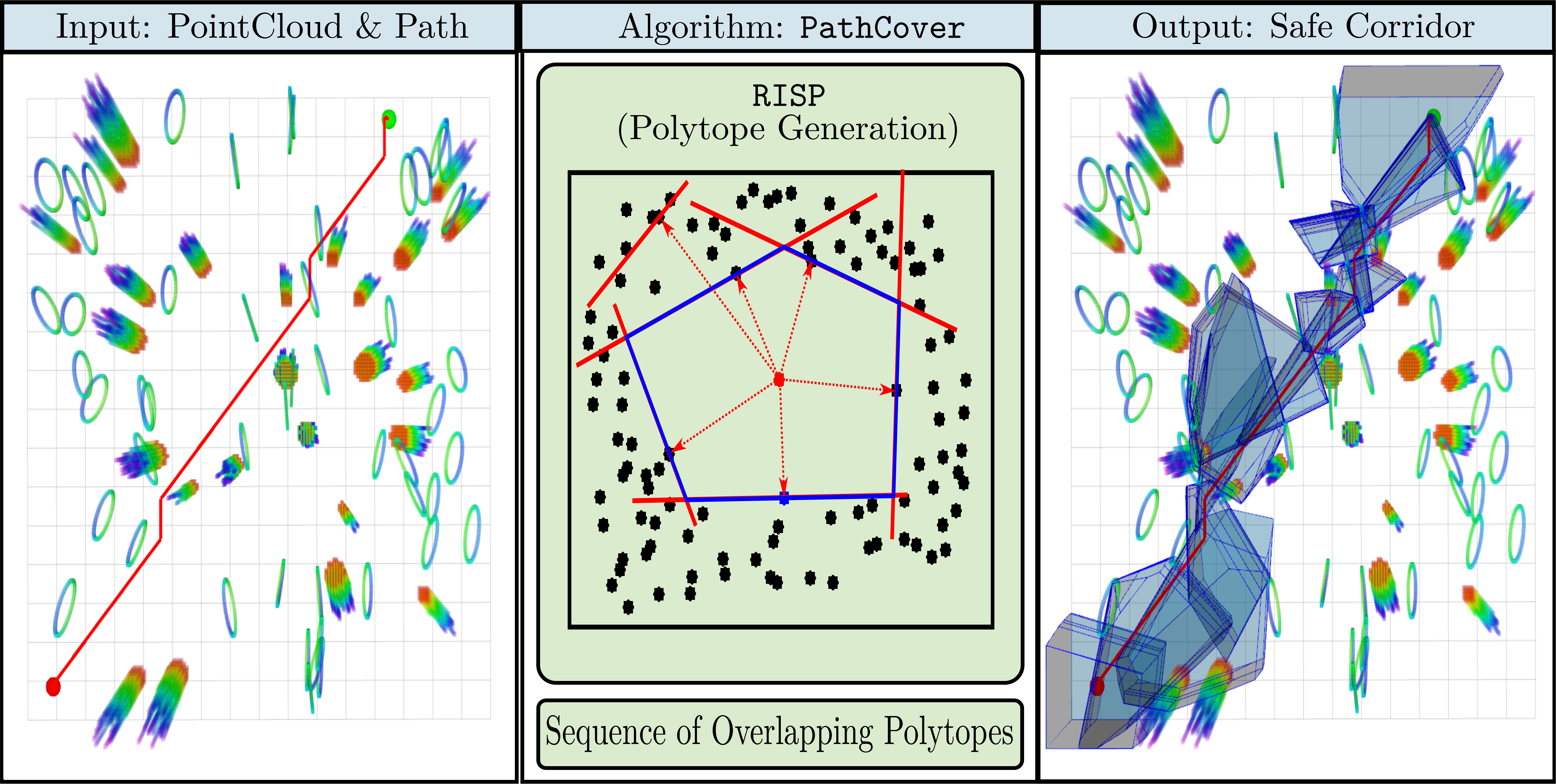}}
\vskip -5pt
\caption{Overview of \texttt{PathCover}: from a raw point cloud
and a reference path (red) connecting the robot to the goal
(green), the method builds overlapping convex polytopes whose
union forms a corridor covering the path.}
\label{fig:RISP_overview}
\vskip -10pt
\end{figure}

Given this constraint, the concept of \emph{\ac{SC}}; a sequence of overlapping polytopes; has become a widely used geometric abstraction in robot motion planning~\cite{Narkh2022RAL, Gao2025TRO}. It decomposes obstacle avoidance into three tractable stages: a global planner that computes a path to the goal, a corridor generator that encloses that path in convex regions~\cite{Deits2015mixedInteger,Liu2017RAL}, and a trajectory optimizer that synthesizes dynamically feasible motion within these regions. The final stage is well established, commonly implemented through mixed-integer programming~\cite{Deits2015mixedInteger}, polynomial spline generation~\cite{Liu2017RAL, Park2022RAL}, or receding-horizon control that treats the corridor as a time-varying state constraint~\cite{How2022FASTER, Gao2022bubble}. This architecture reduces obstacle avoidance, a combinatorially hard problem, to a convex optimization over a polytopic free space, enabling real-time planning on commodity embedded hardware. Once the optimization reduces to a well-posed convex problem, the throughput, reliability, and responsiveness of the overall system are governed primarily by a single upstream component: the size, coverage, and generation speed of the corridor.

In practice, corridor-generation algorithms often fail to keep pace with sensor-rate replanning. Optimization-centric methods produce large-volume, high-quality free-space but at significant computational cost. \ac{IRIS}~\cite{Deits2015IRIS} inflates ellipsoids via \ac{SDP} toward maximal coverage, yet its iterative inflation is too slow for dynamic environments. Configuration space extensions~\cite{Tedrake2023IRISNLP, Tedrake2024IJRR, Tedrake2024IRISZO} improve the results but retain the expensive optimization inner loop. \ac{FIRI}~\cite{Gao2025TRO} reformulates the geometry yet remains limited by repeated \ac{SDP} evaluations. \ac{CIRI}~\cite{ren2025CIRI} carries this line into the robot configuration space and, distinctively, operates directly on raw point clouds without prior map inflation, but it inherits FIRI's iterative conic-optimization inner loop and thus the same per-cycle latency. Parallel efforts to optimize the convex cover directly face the same latency limits~\cite{Kumar2025RAL}.
Optimization-free methods trade this cost at the expense of other factors. Decomp~\cite{Liu2017RAL} avoids optimization by inflating one polytope per piecewise-linear segment, but this couples the polytope count to the path resolution and fragments the corridor. Stereographic-projection and convex-hull constructions~\cite{Sergei2017stereo, Gao2020pointcloud, Gao2022ICRA} require processing the full point cloud~\cite{quickhull}, and segment-sorting schemes~\cite{Kumar2018JMR} scale poorly as sensor density increases. Across these frameworks, achieving runtime behavior that is sufficiently predictable for reliable synchronization with a high-rate control cycle remains challenging.

A parallel line of work relies on voxel-based decomposition~\cite{Gao2020TeachRepeat, Lambert2022VoxelGrid}, aggregating raw sensor streams into structured volumetric maps, such as OctoMaps~\cite{hornung13auro}, and extracting corridors through grid-level geometry~\cite{Gao2018ICRA}. For global planning and multi-query reuse, such maps are highly valuable. Over long missions, however, this very strength becomes a drawback: as occupied cells accumulate during exploration, corridor-extraction latency scales with the aggregate global map. Fine resolutions, which are needed to capture narrow regions, often require power-intensive GPU acceleration~\cite{Lambert2021GPU}, whereas coarser resolutions may fail to properly represent these regions. As a result, the longer a mission runs, the larger the map grows and the slower corridor generation becomes introducing a bottleneck that limits the scalability of long-term autonomous deployments.

\begin{table*}[t!]
\centering
\caption{Methodological comparison of corridor and region generators. Guarantees are relative to each method's stated input obstacle representation.}
\label{tab:related_work}
\setlength{\tabcolsep}{4pt}
\begin{tabular}{p{0.12\textwidth}p{0.16\textwidth}p{0.13\textwidth}p{0.27\textwidth}p{0.22\textwidth}}
\toprule
Method & Obstacle input & Region computation & Construction principle & Relevant tradeoff or guarantee \\
\midrule
\texttt{PathCover} (proposed method) & Point cloud & Randomized geometric operations & Sample a point, add a separating hyperplane, eliminate excluded points, and chain regions along a reference path & Point-cloud separated polytopes and path coverage; conditional expected $O(n)$ and worst-case $O(n^2)$ sampling and elimination \\
\texttt{Decomp}~\cite{Liu2017RAL} & Point cloud & Deterministic geometric operations & Construct a convex region for each segment of the reference path & Number of convex regions is tied to the number of path segments \\
\texttt{Galaxy}~\cite{Gao2020pointcloud} & Point cloud & Convex-hull-based operations & Construct regions through stereographic projection and convex-hull operations & No iterative region optimization; performs a nonlinear transform and takes convex hulls to compute polytopes \\
\texttt{FIRI}/\texttt{CIRI}~\cite{Gao2025TRO,ren2025CIRI} & Point cloud & Iterative optimization & Inflate and refine a convex region from an initial seed or skeleton & Larger free regions; higher computation cost \\
\texttt{IRIS}~\cite{Deits2015IRIS} & Point cloud / convex obstacle enclosure & Alternating optimization & Alternate separating hyperplanes and inscribed-ellipsoid growth & Larger free regions; higher computation cost \\
\bottomrule
\end{tabular}
\end{table*}

To resolve this computational bottleneck, we propose \emph{Randomized Iterative Space Partitioning} \texttt{RISP} and \texttt{PathCover}; see Fig.~\ref{fig:RISP_overview} for an overview. \texttt{RISP} is a randomized point-cloud-based geometric construction that operates on a finite set of depth-sensor points or occupied-voxel samples by iteratively sampling points, constructing separating hyperplanes, and pruning the excluded samples. Building on \texttt{RISP}, \texttt{PathCover} applies this construction sequentially along a reference path to generate a sequence of overlapping polytopes that covers the path. Our approach introduces a unified framework that integrates sampled-point separation, iterative elimination of finite obstacle points, and path-wise chaining of the resulting polytopes, with accompanying finite-coverage and conditional-complexity guarantees. Under a probabilistic condition corresponding to sufficiently frequent fractional elimination of the remaining points, \texttt{RISP} achieves expected linear-time complexity. In particular, motivated by the requirements of deployable autonomous robot navigation, this paper makes the following contributions:
\begin{itemize}
    \item \texttt{RISP}, a randomized algorithm that constructs a convex polytope strictly separated from a finite point set through sampled-point separation and iterative elimination;
    \item \texttt{PathCover}, which chains \texttt{RISP} polytopes along a reference path, with a proof of finite termination, sequential intersection, and complete path coverage under the stated assumptions;
    \item a conditional expected $O(n)$ and unconditional worst-case $O(n^2)$ analysis of the \texttt{RISP} sampling-and-elimination stage, with redundant-facet removal accounted for separately; and
    \item a comparison with state-of-the-art corridor generators~\cite{Liu2017RAL, Gao2025TRO, ren2025CIRI, Gao2020pointcloud}, and demonstrations in ROS/Gazebo simulations and on hardware using the Ghost Robotics Vision60 quadruped\footnote{\url{https://www.ghostrobotics.io/vision-60}}.
\end{itemize}
Finally, Table~\ref{tab:related_work} highlights the methodological distinctions between the proposed \texttt{RISP}/\texttt{PathCover} framework and existing corridor-generation methods.

\section{Problem Formulation and Guarantee Scope}
\label{sec:path_coverage}

Consider a robot navigating within a workspace $\mathcal{W} \subset \mathbb{R}^d$ (where $d=2$ or $3$) filled with obstacles $\mathcal{O} \subset \mathcal{W}$ and let $\mathcal{F} =\mathcal{W} \setminus \mathcal{O}$ be the obstacle-free space. The robot travels from initial position $\mathrm{y}_\mathrm{init} \in \mathcal{F}$ to goal position $\mathrm{y}_\mathrm{goal} \in \mathcal{F}$ and perceives obstacles in $\mathcal{O}$ through sensor information as a finite point cloud:
\begin{equation}\label{eq:point_cloud}
    \mathcal{P} = \left\{ \mathrm{p}_i \in \partial\mathcal{O}\subset\mathbb{R}^{d} \mid i = 1, 2, \dots, n \right\}
\end{equation}
where each point $\mathrm{p}_i$ represents a sampled location on the boundary $\partial\mathcal{O}$ of $\mathcal{O}$, typically obtained from 3-D sensors such as LiDAR, RGB-D cameras, or stereo vision systems.

\textbf{Path Coverage Problem:} 
Given a collection of waypoints $\Pi = \{\mathrm{y}_0, \ldots, \mathrm{y}_L\} \subset \mathcal{F}$, with $\mathrm{y}_0 = \mathrm{y}_\mathrm{init}$ and $\mathrm{y}_L = \mathrm{y}_\mathrm{goal}$, whose piecewise-linear interpolation $\Gamma(\Pi)$ lies in the interior of $\mathcal{F}$, and a finite point set $\mathcal{P}$ representing observed obstacles, generate $J \geq 1$ sequentially intersecting polytopes\footnote{A polytope is a bounded intersection of finitely many closed halfspaces and is therefore convex and compact.} $\mathcal{G} = \{\mathcal{H}_1, \ldots, \mathcal{H}_J\}$,
\begin{equation}\label{eq:poly-def}
    \mathcal{H}_j = \left\{ \mathrm{y} \in \mathbb{R}^d \mid A_j\mathrm{y} \leq b_j,~A_j \in \mathbb{R}^{n_j \times d}, b_j \in \mathbb{R}^{n_j} \right\}
\end{equation}
with $n_j \geq d+1,~ j \in \{1,\ldots,J\}$ satisfying the following \emph{path coverage conditions} (PCC): 
\begin{enumerate}[label=PCC.\arabic*),leftmargin=*, labelsep=0.5em]
    \item \label{PCC:1} each polytope $\mathcal{H}_j$, $j \in \{1,\ldots,J\}$, is strictly separated from every point in $\mathcal{P}$;
    \item \label{PCC:2} $\mathcal{H}_j \cap \mathcal{H}_{j+1} \neq \emptyset$, $j \in \{1,\ldots,J-1\}$ (sequentially intersected polytopes)
    \item \label{PCC:3} $\Gamma(\Pi) \subset \bigcup_{j=1}^{J} \mathcal{H}_j$ (the path is covered)  
\end{enumerate}
These sequentially intersecting polytopes form a safe corridor with respect to the point-cloud representation and can be provided to a downstream planner as a set of convex constraints~\cite{Deits2015mixedInteger, Liu2017RAL, Gao2018ICRA}.

\begin{rem} \label{rem:point-cloud-representation}
    The finite set $\mathcal{P}$ is the sole obstacle representation assumed to be available. Consequently, the separation guarantees required by the Path Coverage Problem defined above are with respect to this representation; i.e., $\mathcal{H}_j \cap \mathcal{P} = \emptyset$, which  guarantees separation from the supplied point-cloud representation. The term ``safe corridor'' is used in this sense throughout the paper. When implementing our approach in practice, to obstacle-cluttered spaces and finite-size robots, the corresponding geometric margins can readily be incorporated to account for sensing uncertainty, tracking errors, and the robot footprint.
\end{rem}

\subsection{The PathCover Algorithm}
\label{subsec:PathCover_alg}

Algorithm~\ref{alg:PathCover} outlines the \texttt{PathCover} algorithm for constructing safe corridors along a predefined reference path, represented as a collection of waypoints $\Pi = \{\mathrm{y}_0, \ldots, \mathrm{y}_L\} \subset \mathcal{F}$ within a point cloud environment $\mathcal{P} \subset \mathcal{W}$. The reference path can be generated using any sampling-based or graph-based planning method~\cite{LaValle2006Planning}; in our implementation, we use \ac{JPS}~\cite{Harabor2011jps}. The algorithm generates a sequence of polytopes $\mathcal{G}$ that are strictly separated from the points in $\mathcal{P}$ and cover the reference path for use in downstream motion planning.

\begin{algorithm}[b!]
\caption{$\texttt{PathCover}(\Pi$, $\mathcal{P}, \alpha)$}\label{alg:PathCover}
\begin{algorithmic}[1]
\State \textbf{Inputs}: $\Pi=\{\mathrm{y}_0,..., \texttt{y}_L\} \subset \mathcal{F}$, $~\mathcal{P} \subset \mathcal{W}$, $~\alpha \in (0, 1)$
\State \textbf{Outputs}: Polytope list $\mathcal{G}$
\State $\mathcal{H}_\mathrm{new} \gets \texttt{RISP}(\mathrm{y}_0,~\mathcal{P}, \alpha)$, $\ell \gets 0$
\State Initialize list: $\mathcal{G}.\texttt{append}(\mathcal{H}_\mathrm{new})$
    \While{$\mathrm{y}_\mathrm{goal}\notin \mathcal{H}_\mathrm{new}$ OR $\ell < L$} 
        \If{$\mathrm{y}_\ell \in \mathcal{H}_\mathrm{new}$}
            \State $\mathrm{y}_\mathrm{in} \gets \mathrm{y}_{\ell}$
            \State $\ell \gets \ell+1$
            \State \textbf{continue}
        \Else
            \State $\mathrm{y}_\mathrm{in} \gets \texttt{PolyLineIntersect}(\mathcal{H}_\mathrm{new}, \mathrm{y}_\mathrm{in}, \mathrm{y}_{\ell})$
            \State $\mathcal{H}_\mathrm{new} \gets \texttt{RISP}(\mathrm{y}_\mathrm{in}, \mathcal{P}, \alpha)$
                \State $\mathcal{G}.\texttt{append}(\mathcal{H}_\mathrm{new})$
        \EndIf
    \EndWhile
\State \Return $\mathcal{G}$
\end{algorithmic}
\end{algorithm}
\begin{algorithm}[b!]
\caption{$\texttt{RISP}(\mathrm{y}_\mathrm{seed}$, $\mathcal{P}$, $\alpha)$}\label{alg:risp}
\begin{algorithmic}[1]
\State \textbf{Input}: $\mathrm{y}_\mathrm{seed}$, $\mathcal{P}$, $\alpha \in (0, 1)$
\State \textbf{Output}: $\mathcal{H}$
\While{$\mathcal{P}$ is not empty}
    \State Randomly select a point $\mathrm{p}$ from $\mathcal{P}$
    \State Compute $\mathrm{a}$ and $\mathrm{b}$ according to \eqref{eq:plane_selection}
    \For{$\forall \mathrm{q} \in \mathcal{P}$}
        \If{$\mathrm{a}^\top \mathrm{q} > \mathrm{b}$}
            \State Remove $\mathrm{q}$ from $\mathcal{P}$
        \EndIf
    \EndFor
    \State $\mathcal{H}.\texttt{append}((\mathrm{a}, \mathrm{b}))$
\EndWhile
\State Append workspace boundary constraints to $\mathcal{H}$
\State Remove redundant hyperplanes from $\mathcal{H}$ using a fixed-dimensional convex-hull routine (\texttt{Quickhull} is used)
\State \Return $\mathcal{H}$
\end{algorithmic}
\end{algorithm}

The algorithm begins by invoking the \texttt{RISP} function (Algorithm~\ref{alg:risp}, detailed in Section~\ref{sec:RISP}) at the starting point $\mathrm{y}_0$ to generate the initial convex polytope $\mathcal{H}_\mathrm{new}$ that is strictly separated from the points in $\mathcal{P}$. This polytope is appended to the corridor list $\mathcal{G}$, and the algorithm proceeds further by iterating over waypoints in the reference path. The main loop continues until the final waypoint $\mathrm{y}_L = \mathrm{y}_\mathrm{goal}$ lies within the most recently generated polytope\footnote{We describe the uncapped algorithm and in Section~\ref{sec:properties} below study its theoretical properties; if needed in practice, the algorithm may be terminated after a prescribed number of iterations (see Section~\ref{sec:sim_results}).} $\mathcal{H}_\mathrm{new}$. At each iteration, the algorithm processes the waypoints in $\Pi$ and checks whether the current waypoint $\mathrm{y}_\ell$ lies within $\mathcal{H}_\mathrm{new}$. If so, the waypoint is already covered, the index is incremented to the next waypoint, and the loop proceeds. If $\mathrm{y}_\ell$ lies outside $\mathcal{H}_\mathrm{new}$, the algorithm computes the intersection point between the current corridor and the line segment connecting the last known interior point $\mathrm{y}_\mathrm{in}$ and the current point $\mathrm{y}_\ell$ using the \texttt{PolyLineIntersect} function. This intersection point serves as the seed for constructing the next polytope in the corridor, which is generated using \texttt{RISP} and added to $\mathcal{G}$. The end result is a sequence of overlapping polytopes that collectively cover the reference path and are strictly separated from the point  cloud. Fig.~\ref{fig:corridor_generation_path} illustrates this incremental process.

\begin{figure*}[t!]
\vskip +5pt
\centering
\subfigure[]{\includegraphics[width=0.23\textwidth]{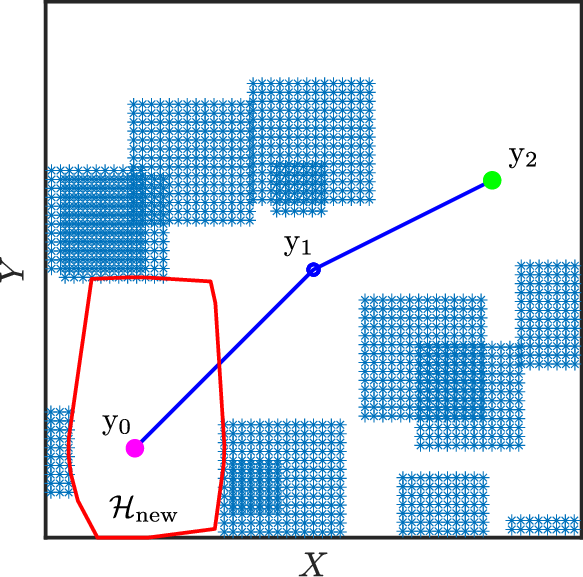}
\label{fig:corridor_generation_path_1}}
\centering
\subfigure[]{\includegraphics[width=0.23\textwidth]{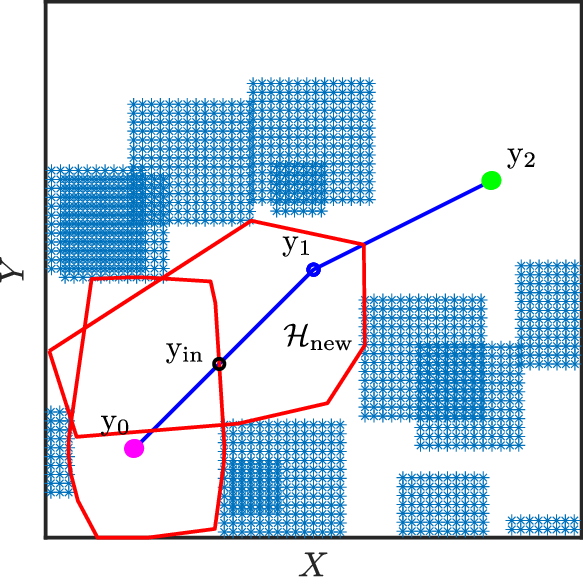}
\label{fig:corridor_generation_path_2}}
\centering
\subfigure[]{\includegraphics[width=0.23\textwidth]{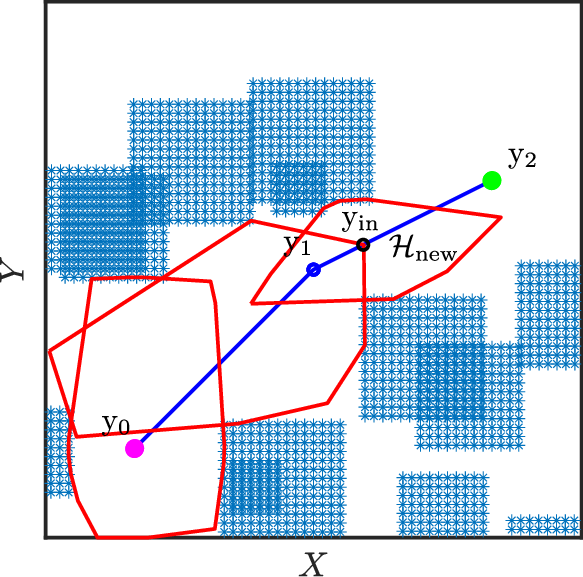}
\label{fig:corridor_generation_path_3}}
\centering
\subfigure[]{\includegraphics[width=0.23\textwidth]{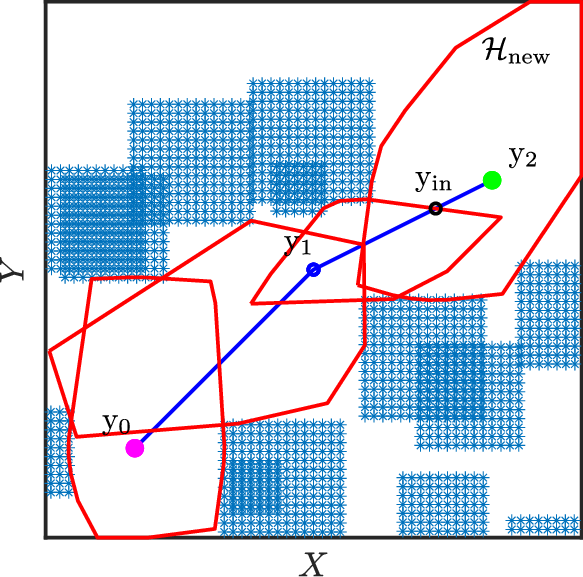}
\label{fig:corridor_generation_path_4}}
\vskip -5pt
\caption{Incremental application of \texttt{RISP} along a path to generate corridor using Algorithm~\ref{alg:PathCover}. The path (blue) is $\Pi = \{\mathrm{y}_0, \mathrm{y}_1, \mathrm{y}_2\}$. The algorithm begins with $\mathrm{y}_0$ as seed to generate the first polytope $\mathcal{H}_\mathrm{new}$ (a). It then computes the intersection point $\mathrm{y}_\mathrm{in}$ between the polytope and the line segment connecting $\mathrm{y}_1$ and $\mathrm{y}0$, uses this point as a seed, and generates another polytope (b–c). The algorithm continues until the goal is inside the last generated polytope (d).} 
\label{fig:corridor_generation_path}
\vskip -15pt
\end{figure*}

\subsection{The \texttt{RISP} Algorithm}
\label{sec:RISP}

Here we describe the \texttt{RISP} algorithm (Algorithm~\ref{alg:risp}). \texttt{RISP} is a randomized algorithm that constructs polytopes directly from unordered point cloud data. This is done by iteratively sampling and eliminating points from the point cloud until no such points remain around a given seed point $\mathrm{y}_\mathrm{seed}$. In more detail, at each iteration, the algorithm samples an obstacle point $\mathrm{p}$ from the point cloud $\mathcal{P}$ and constructs the hyperplane 
\begin{equation} \label{eq:hyperplane}
    h = \{\mathrm{x} \in \mathbb{R}^d \mid \mathrm{a}^{\top}\mathrm{x} = \mathrm{b},~ \mathrm{a} \in \mathbb{R}^d,~ \mathrm{b} \in \mathbb{R} \}
\end{equation}
with 
\begin{equation} \label{eq:plane_selection}
    \mathrm{a} = \mathrm{p} - \mathrm{y}_\mathrm{seed} ~~~\text{and}~~~
    \mathrm{b} = \mathrm{a}^\top \mathrm{y}_\mathrm{seed} + (1-\alpha) \| \mathrm{a} \|^2_2 \enspace.
\end{equation}
By construction, the hyperplane $h$ is orthogonal to the vector $\mathrm{a}$ connecting the seed point $\mathrm{y}_\mathrm{seed}$ with the sampled point $\mathrm{p}$ and passes through the point $\mathrm{y}_\star = \alpha \mathrm{y}_\mathrm{seed} + (1-\alpha) \mathrm{p}$, where $\alpha \in (0,1)$ is a small parameter selected by the user. Furthermore, $h$ strictly separates the sampled point $\mathrm{p}$ from the (open) ball of radius $\rho = (1-\alpha) \|\mathrm{a}\|_2$ that is centered at the seed point $\mathrm{y}_\mathrm{seed}$. Increasing $\alpha$ moves the separator toward the seed and reduces the guaranteed radius, thus controlling the tradeoff between polytope size and point elimination. The algorithm then discards all obstacle points that lie on the side of $h$ opposite to $\mathrm{y}_\mathrm{seed}$. 

This procedure repeats by selecting a new obstacle point from the \emph{remaining} points in $\mathcal{P}$ until no points remain, thus yielding a set of half-space constraints that collectively separate $\mathrm{y}_\mathrm{seed}$ from the obstacle point cloud $\mathcal{P}$. Finally, the boundary constraints of the workspace $\mathcal{W}$ are appended and redundant constraints are removed. The latter is done by first mapping each half-space $\mathrm{a}_m^\top\mathrm{x} \leq \mathrm{b}_m$ to a point $\mathrm{d}_m = \mathrm{a}_m / (\mathrm{b}_m - \mathrm{a}_m^\top \mathrm{y}_\mathrm{seed})$ in the dual space; this mapping is well defined since the seed point $\mathrm{y}_\mathrm{seed}$ is in the interior of the primal polytope, implying that $\mathrm{b}_m-\mathrm{a}_m^\top \mathrm{y}_\mathrm{seed}>0$. This way, facets of the primal polytope correspond to vertices $\{\mathrm{d}_m\}$ in the dual space. Then, redundant constraints are eliminated by computing the convex hull of $\{\mathrm{d}_m\}$ and retaining only those constraints whose dual points lie on the hull. Any algorithm for computing fixed-dimensional convex hulls can be used; in our implementation we use the \texttt{Quickhull}\footnote{http://www.qhull.org/} algorithm~\cite{quickhull}. 

\section{Properties of the Algorithm}
\label{sec:properties}

This section begins by analyzing the completeness of \texttt{PathCover} followed by time complexities of the \texttt{PathCover} and \texttt{RISP} algorithms.

\subsection{Completeness of the PathCover Algorithm}

We now establish the completeness of the \texttt{PathCover} algorithm. Specifically, we state here and prove in Appendix~\ref{app:theorem1} that for any obstacle-free path $\Pi$, the algorithm is guaranteed to terminate in finite time and to produce a sequence of mutually intersecting convex polytopes that fully covers the path. To this end, we make the following assumptions.

\begin{assump} \label{ass:free_space}
The workspace $\mathcal{W} \subset \mathbb{R}^d$ is a polytope and the obstacle set $\mathcal{O} \subset \mathcal{W}$ is closed in $\mathcal{W}$.
\end{assump}

\begin{assump} \label{ass:free_path}
The line segment connecting any two successive waypoints $\mathrm{y}_{\ell-1}$ and $\mathrm{y}_\ell$ in the path $\Pi$ lies entirely in the interior of the free space $\mathcal{F} = \mathcal{W} \setminus \mathcal{O}$. 
\end{assump}
Both assumptions are commonly satisfied in practical settings. Assumption~\ref{ass:free_space} is naturally satisfied when operating within a bounding box introduced for computational purposes. It implies that the obstacle-free region $\mathcal{F}$ is open relative to $\mathcal{W}$. Assumption~\ref{ass:free_path} ensures that, for every point on the line segment connecting $\mathrm{y}_{\ell-1}$ and $\mathrm{y}_\ell$, there exists a ball centered at that point that is entirely contained in $\mathcal{F}$. With these assumptions, we now state the following theorem, which establishes that the collection of waypoints and the piecewise-linear path connecting them can be covered by finitely many polytopes. 

\begin{theorem} \label{thm:pathcover_completeness}
Let $\mathcal{W} \subset \mathbb{R}^d$ and $\mathcal{O} \subset \mathcal{W}$ satisfy Assumption~\ref{ass:free_space} and define the free space $\mathcal{F} = \mathcal{W} \setminus \mathcal{O}$. Let $\mathcal{P}$ be a finite set of points sampled from the boundaries of the obstacles $\mathcal{O}$. Let $\Pi \!=\! \{\mathrm{y}_0, \dots, \mathrm{y}_L\} \!\subset\! \mathcal{F}$ be a collection of waypoints satisfying Assumption~\ref{ass:free_path}. Then, the \texttt{PathCover} algorithm (Algorithm \ref{alg:PathCover}) computes a finite sequence of polytopes $\mathcal{G} = \{\mathcal{H}_1, \dots, \mathcal{H}_J\}$ satisfying the path coverage conditions \ref{PCC:1}-\ref{PCC:3}.
\end{theorem}

A proof of Theorem~\ref{thm:pathcover_completeness} is provided in Appendix~\ref{app:theorem1}. The proof shows that Algorithm~\ref{alg:PathCover} bounds the number $J$ of polytopes required to cover the path in terms of the path's geometry and clearance, rather than its combinatorial resolution. This contrasts with state-of-the-art segment-based decomposition methods such as \texttt{Decomp}~\cite{Liu2017RAL}, which construct one polytope per piecewise-linear segment and thus fix $J$ to the number of segments by design. \texttt{PathCover}, instead, allows a single polytope to span multiple consecutive segments whenever the local clearance permits. This flexibility directly reduces the number of constraints passed to the downstream planner at each replanning step.

\begin{rem} \label{rem:pathcover_practice}
Although Theorem~\ref{thm:pathcover_completeness} guarantees coverage of the entire path $\Gamma(\Pi)$, \texttt{PathCover} need not be run over the full path (which can be very long) in a single call; in practice, it can be applied in a receding-horizon fashion to the leading portion of $\Pi$ within the latest sensor measurement, possibly capping the horizon length for computational efficiency (cf. Section~\ref{sec:sim_results}). The completeness guarantees of Theorem~\ref{thm:pathcover_completeness} still apply to this local sub-path.
\end{rem}

\subsection{Complexity Analysis of the \texttt{RISP} Algorithm}
\label{subsec:complexity_analysis}

The complexity of \texttt{RISP} depends on the rate at which the point cloud $\mathcal{P}$ shrinks across iterations. To formalize this, set $\mathcal{P}_0 = \mathcal{P}$ and denote its cardinality by $N_0 = |\mathcal{P}_0|$. Given a seed point $\mathrm{y}_\mathrm{seed}$ in the interior of $\mathcal{F}$, \texttt{RISP} samples $\mathrm{p}_0 \in \mathcal{P}_0$ (Algorithm~\ref{alg:risp}, line 4), defines the hyperplane $h_0 = \{\mathrm{x} \in \mathbb{R}^d \mid \mathrm{a}_0^\top \mathrm{x} = \mathrm{b}_0 \}$ via~\eqref{eq:plane_selection}, and discards $\mathcal{Q}_0 = \{\mathrm{q}\in\mathcal{P}_0 \mid \mathrm{a}_0^\top \mathrm{q} > \mathrm{b}_0 \}$ (Algorithm~\ref{alg:risp}, lines 7--9), yielding the set $\mathcal{P}_1 = \mathcal{P}_0 \setminus \mathcal{Q}_0$ of remaining points. Iteratively, for each $m \geq 1$, let $\mathcal{P}_m \subset \mathcal{P}_0$ denote the remaining points. A hyperplane $h_m = \{\mathrm{x} \in \mathbb{R}^d \mid \mathrm{a}_m^\top \mathrm{x} = \mathrm{b}_m \}$ is constructed via~\eqref{eq:plane_selection} from the (uniformly) sampled point $\mathrm{p}_m \in \mathcal{P}_m$, and the set $\mathcal{Q}_m = \{\mathrm{q} \in \mathcal{P}_m \mid \mathrm{a}_m^\top \mathrm{q} > \mathrm{b}_m \}$ is removed from $\mathcal{P}_m$. Clearly, for a fixed seed $\mathrm{y}_\mathrm{seed}$, $h_m$ depends on $\mathrm{p}_m$, and thus the size of the eliminated set $\mathcal{Q}_m$ is determined by $\mathrm{p}_m$. The key question is therefore how often a uniformly sampled point $\mathrm{p}_m$ induces a hyperplane $h_m$ that removes a nontrivial fraction of $\mathcal{P}_m$.

This question is particularly relevant for the dense, locally coplanar point clouds produced by LiDAR and depth cameras, where samples cluster into extended near-planar regions. In such settings, a hyperplane passing near a single sample typically removes an entire neighborhood from $\mathcal{P}_m$. The following theorem formalizes this intuition and derives the resulting complexity of \texttt{RISP}; Section~\ref{subsec:empirical_support} provides empirical validation on both simulated and real-world datasets.

\begin{theorem}\label{thm:RISP_complexity}
Let $\mathcal{W}$, $\mathcal{O}$, $\mathcal{F}$ and $\mathcal{P}$ be as in Theorem~\ref{thm:pathcover_completeness}. Fix a seed $\mathrm{y}_\mathrm{seed}$ in the interior of $\mathcal{F}$, and set $\mathcal{P}_0 = \mathcal{P}$ with $N_0 = |\mathcal{P}_0| = n$. Apply \texttt{RISP}. 
For each iteration $m \geq 0$, let $\mathcal{P}_m \subset \mathcal{P}_0$ denote the remaining set of points with $N_m = |\mathcal{P}_m|$. Given a sampled point $\mathrm{p}_m$, define $\mathcal{P}_{m+1} = \mathcal{P}_m \setminus \mathcal{Q}_m$, where $\mathcal{Q}_m = \{\mathrm{q} \in \mathcal{P}_m \mid \mathrm{a}_m^\top \mathrm{q} > \mathrm{b}_m \}$ is the set of points removed by the hyperplane $h_m$ defined by~\eqref{eq:hyperplane}--\eqref{eq:plane_selection}.
Assume there exist constants $\beta \in (0,1)$ and $r \in (0,1]$, such that, for all $m < \tau = \inf\{m \geq 0 \mid N_m = 0\}$,
\begin{equation}\label{eq:beta-good}
  \Pr \left\{ N_{m+1} \leq (1-\beta)\,N_m \mid \mathcal{P}_m\right\} \geq r
\end{equation}
where the probability is taken over the uniform choice of $\mathrm{p}_m \in \mathcal{P}_m$. Then, \texttt{RISP} has expected time complexity $O(n)$ and worst-case time complexity $O(n^2)$.
\end{theorem}

A proof of Theorem~\ref{thm:RISP_complexity} is presented in Appendix~\ref{app:theorem2}. To build intuition,~\eqref{eq:beta-good} states that, with probability at least $r$, a uniformly sampled point $\mathrm{p}_m$ induces a hyperplane $h_m$ that removes at least a fraction $\beta$ of the remaining points at iteration $m$. Essentially, $\beta$ quantifies \emph{how aggressive} a successful elimination is and $r$ measures \emph{how often} it occurs; their product $\beta r$ represents the effective drift rate that results in a geometric decay of the remaining number of points $N_m$ in the point cloud. In general, $\beta$ and $r$ depend on the seed $\mathrm{y}_\mathrm{seed}$ and the geometry of $\mathcal{P}$. The condition~\eqref{eq:beta-good} is particularly relevant for structured point clouds, such as locally coplanar LiDAR or depth-camera clusters, where hyperplanes induced by uniformly sampled points tend to discard a significant portion of the cluster. In such cases, both $\beta$ and $r$ can be bounded away from zero. Section~\ref{subsec:empirical_support} provides empirical estimates of these constants on three datasets and confirms the predicted geometric decay of $N_m$ from Theorem~\ref{thm:RISP_complexity}.

\begin{rem} \label{rem:worst_case}
The $O(n^2)$ worst case arises only under a highly contrived geometric configuration in which each sample eliminates exactly one point (the sampled point), a scenario that is essentially absent in real point-cloud data. In all evaluated configurations, \texttt{RISP} exhibits the linear expected performance predicted by Theorem~\ref{thm:RISP_complexity} (see Fig.~\ref{fig:complexity_decay}). 
\end{rem}

\subsection{Complexity Analysis of the PathCover Algorithm}

The following theorem characterizes the expected and worst-case computational complexities of the \texttt{PathCover} algorithm.

\begin{theorem}\label{thm:pathcover_complexity}
Let $\mathcal{W}$, $\mathcal{O}$, $\mathcal{F}$, $\Pi$ and $\mathcal{P}$ be as in Theorem~\ref{thm:pathcover_completeness} and define $n = |\mathcal{P}|$. With the notation and assumptions of Theorem~\ref{thm:RISP_complexity}, the \texttt{PathCover} algorithm (Algorithm \ref{alg:PathCover}) has expected time complexity $O(n)$ and worst-case time complexity $O(n^2)$.
\end{theorem}

A proof of this theorem is provided in Appendix~\ref{app:theorem3}. We briefly note here that the total computational cost of the algorithm is driven by two components: the $J$ sequential \texttt{RISP} calls, one for each polytope in the collection $\mathcal{G} = \{\mathcal{H}_1, ..., \mathcal{H}_J \}$ returned by \texttt{PathCover} (Algorithm~\ref{alg:PathCover}, line 15), and the per-polytope membership and intersection checks (Algorithm~\ref{alg:PathCover}, lines 5 and 6), whose cost scales linearly  with the size of the current polytope $\mathcal{H}_\mathrm{new}$.

\section{Performance Evaluation and Benchmarking}
\label{sec:performance_eval}

We compare \texttt{PathCover} and \texttt{RISP} against state-of-the-art free-space decomposition methods on two complementary benchmarks. The first measures end-to-end path coverage in real-world-scale environments and directly captures the metrics that govern online planning performance; namely, (i) computation time, (ii) coverage volume, and (iii) polytope, facet and constraint counts. The second benchmark isolates per-polytope construction cost over a much denser point cloud than those typically encountered in navigation environments, stress-testing the scalability of \texttt{RISP} itself. The baselines include four state-of-the-art methods: \texttt{Decomp}\footnote{https://github.com/sikang/DecompROS}~\cite{Liu2017RAL}, which inflates axis-aligned boxes iteratively; \ac{FIRI}\footnote{https://github.com/ZJU-FAST-Lab/Aerobatic-Planner (A 3-D implementation is available).}~\cite{Gao2025TRO}, which refines polytopes through repeated halfspace optimization seeded by a skeleton; and \texttt{CIRI}\footnote{https://github.com/hku-mars/SUPER}~\cite{ren2025CIRI}, which extends \ac{FIRI}'s conic-optimization inflation to the configuration space and operates directly on raw point clouds and \texttt{Galaxy}\footnote{https://github.com/StarryN/Galaxy (A 2-D implementation is available).}~\cite{Gao2020pointcloud} which is included in the 2-D experiments as a convex-hull-based alternative. For all baseline methods, we used the publicly available implementations without algorithmic modification. All benchmarks, simulations, and experiments in this and subsequent sections were conducted on an Intel $\mathrm{i}7{\text -}9750\mathrm{H}$ ($2.6$ GHz) machine with $16\mathrm{GB}$ RAM, using C/C++ implementations with $\alpha=0.01$ for \texttt{PathCover}.

\begin{table*}[t!]
\centering
\caption{Path-coverage and downstream trajectory-optimization statistics over 100 randomly generated environments per density.}
\label{tab:performance_metric}
\setlength{\tabcolsep}{3pt}
\begin{tabular}{llccc ccc ccc}
\toprule
\multirow{2}{*}{Metric} & \multirow{2}{*}{Method} &
\multicolumn{3}{c}{Sparse} & \multicolumn{3}{c}{Medium} &
\multicolumn{3}{c}{Dense} \\
\cmidrule(lr){3-5} \cmidrule(lr){6-8} \cmidrule(lr){9-11}
& & Mean & Std. dev. & Max. & Mean & Std. dev. & Max. & Mean & Std. dev. & Max. \\
\midrule
\multirow{4}{*}{Corridor generation time (ms)}
& \texttt{PathCover}  & \textbf{6.57} & \textbf{1.19} & \textbf{10.56}
                 & \textbf{27.83} & \textbf{5.42} & \textbf{43.41}
                 & \textbf{173.58} & \textbf{26.01} & \textbf{251.86} \\
& \texttt{CIRI}       & 119.92 & 27.96 & 188.33
                 & 460.30 & 125.24 & 702.18
                 & 1887.51 & 456.42 & 2765.07 \\
& \texttt{Decomp}         & 163.65 & 39.64 & 264.51
                 & 601.23 & 169.22 & 939.59
                 & 2342.53 & 618.48 & 3640.56 \\
& \ac{FIRI}      & 346.19 & 58.84 & 558.79
                 & 1125.10 & 191.43 & 1566.76
                 & 4139.87 & 771.23 & 6248.95 \\

\cmidrule(lr){1-11}
\multirow{4}{*}{Volume per polytope $(\mathrm{m}^3)$}
& \texttt{PathCover}  & 350.14 & 208.18 & 1361.34
                 & 155.07 & 106.69 & 791.82
                 & 44.11 & 28.64 & 202.90 \\
& \texttt{CIRI}       & 510.38 & 259.47 & 1713.51
                 & 238.35 & 134.61 & 872.76
                 & 79.78 & 106.22 & 3751.71 \\                 
& \texttt{Decomp}         & 473.81 & 242.04 & 1534.17
                 & 223.91 & 133.53 & 881.13
                 & 63.50 & 37.64 & 263.02 \\
& \ac{FIRI}      & \textbf{771.35} & \textbf{380.66} & \textbf{4609.59}
                 & \textbf{367.43} & \textbf{299.22} & \textbf{5012.13}
                 & \textbf{383.60} & \textbf{1047.82} & \textbf{5084.16} \\
\cmidrule(lr){1-11}
\multirow{4}{*}{No. of Polytopes per corridor}
& \texttt{PathCover}  & \textbf{12} & \textbf{2} & \textbf{18}
                 & \textbf{18} & \textbf{3} & \textbf{28}
                 & \textbf{28} & \textbf{4} & \textbf{37} \\
& \texttt{CIRI}       & 42 & 9 & 63
                 & 51 & 14 & 79
                 & 57 & 13 & 81 \\                 
& \texttt{Decomp}         & 42 & 9 & 63
                 & 51 & 14 & 79
                 & 57 & 13 & 81 \\
& \ac{FIRI}      & 51 & 7 & 70
                 & 58 & 10 & 81
                 & 63 & 11 & 85 \\
\cmidrule(lr){1-11}                 
\multirow{4}{*}{No. of facets/constraints per corridor}
& \texttt{PathCover}  & \textbf{183} & \textbf{35} & \textbf{276}
                 & \textbf{312} & \textbf{57} & \textbf{472}
                 & \textbf{512} & \textbf{87} & \textbf{701} \\
& \texttt{CIRI}  & 546 & 125 & 844
                 & 709 & 199 & 1055
                 & 808 & 205 & 1239 \\                 
& \texttt{Decomp}  & 530 & 120 & 833
                   & 698 & 193 & 1049
                   & 802 & 203 & 1258 \\
& \ac{FIRI}      & 1136 & 156 & 1561
                 & 1384 & 247 & 1874
                 & 1484 & 286 & 2210 \\
\cmidrule(lr){1-11}
\multirow{4}{*}{Trajectory optimization \cite{Gao2022Gcopter} time per corridor (ms)}
& \texttt{PathCover}  & \textbf{154.75} & \textbf{41.47} & \textbf{350.39}
                 & \textbf{179.94} & \textbf{48.49} & \textbf{368.99}
                 & \textbf{172.53} & \textbf{41.90} & \textbf{293.40} \\
& \texttt{CIRI}  & 788.44 & 710.99 & 3705.79
                 & 1141.45 & 964.09 & 4212.15
                 & 1421.61 & 749.55 & 3277.36 \\                 
& \texttt{Decomp}  & 874.10 & 660.22 & 3163.41
                   & 1275.27 & 939.92 & 4552.96
                   & 1316.44 & 727.61 & 3474.73 \\
& \ac{FIRI}      & 598.96 & 470.57 & 3111.39
                 & 1439.25 & 1363.05 & 8771.90
                 & 2081.04 & 1475.60 & 6210.36 \\
\bottomrule
\end{tabular}
\end{table*}

\subsection{Path Coverage Performance}

We evaluate path coverage in a $50 \times 50 \times 5$~m workspace under three point-cloud--density categories: \emph{sparse} (60564 points on average), \emph{medium} (170761 points on average), and \emph{dense} (551759 points on average). For each category, we independently generate 100 environments with random obstacle point clouds and randomly sampled start--goal pairs, yielding 300 environments in total. A collision-free reference path is generated using \ac{JPS}~\cite{Harabor2011jps}; instances corresponding to paths shorter than $25\,\mathrm{m}$ are discarded and resampled, so that the resulting corridor sequences correspond to non-trivial navigation tasks. For every retained instance, the same reference path is processed by \texttt{PathCover}, \ac{FIRI}, \texttt{CIRI}, and \texttt{Decomp}, and the resulting corridor is supplied to the same downstream trajectory-optimization procedure; we use \texttt{GCOPTER}~\cite{Gao2022Gcopter}\footnote{https://github.com/ZJU-FAST-Lab/GCOPTER}. For each method, Table~\ref{tab:performance_metric} reports corridor-generation time, per-polytope volume, polytope and facet counts per corridor, and trajectory-optimization time.

Table~\ref{tab:performance_metric} reveals a consistent pattern: \texttt{PathCover} achieves the lowest computation time and produces the most compact corridor representation across all obstacle-density levels, while \ac{FIRI} generates larger individual polytopes at substantially higher computational cost. \texttt{PathCover}'s mean polytope count is approximately $71\%$, $64\%$, and $49\%$ lower than the smallest competing mean in the sparse, medium, and dense categories, respectively. Its corresponding mean facet/constraint count is approximately $65\%$, $55\%$, and $36\%$ lower than the smallest competing mean. Since the facets define the half-space constraints of the corridor, this reduction also decreases the number of geometric constraints supplied to the downstream trajectory optimizer.

Table~\ref{tab:performance_metric} also shows that \texttt{PathCover} generates smaller individual regions. Specifically, the ratio of \ac{FIRI}'s mean per-polytope volume to that of \texttt{PathCover} increases from $2.2\times$ in sparse environments to $8.7\times$ in dense environments. This indicates that \texttt{PathCover} achieves path coverage through more compact corridors, using fewer polytopes and constraints rather than relying on larger individual regions. Such compactness is advantageous for downstream trajectory optimization, since the computational burden is influenced not only by the size of the free-space regions but also by the number of polytopes and associated half-space constraints.

This advantage is also evident in the downstream trajectory-optimization computation time. The mean optimization time with \texttt{PathCover} is $154.75$~ms, $179.94$~ms, and $172.53$~ms in sparse, medium, and dense environments, respectively. Relative to the fastest competing mean in each category, this corresponds to speedups of $3.9\times$, $6.3\times$, and $7.6\times$; compared with all individual baselines, the observed speedups range from $3.9\times$ to $12.1\times$. In dense environments, \texttt{PathCover} reduces the mean optimization time from $1316.44$--$2081.04$~ms for the baselines to $172.53$~ms. Hence, the advantage of \texttt{PathCover} extends beyond faster corridor construction: by covering the reference path with fewer polytopes and facets, it reduces the constraint burden on the trajectory optimizer and consistently lowers downstream optimization time across the tested environments.

\subsection{Synthetic Stress Test: Individual Polytope Scalability}

\begin{table*}[t!]
\centering
\caption{Individual-polytope statistics over 500 independently generated environments per dimension and density.}
\label{tab:performance_metrics}
\setlength{\tabcolsep}{3pt}
\begin{tabular}{lllccc ccc ccc}
\toprule
\multirow{2}{*}{Dim.} & \multirow{2}{*}{Metric} & \multirow{2}{*}{Method}
& \multicolumn{3}{c}{Sparse} & \multicolumn{3}{c}{Medium}
& \multicolumn{3}{c}{Dense} \\
\cmidrule(lr){4-6} \cmidrule(lr){7-9} \cmidrule(lr){10-12}
& & & Mean & Std. dev. & Max. & Mean & Std. dev. & Max. & Mean & Std. dev. & Max. \\
\midrule
\multirow{9}{*}{2-D}
& \multirow{3}{*}{Time (ms)}
& \texttt{RISP} & \textbf{0.35} & \textbf{0.12} & \textbf{0.91} &
                 \textbf{4.08} & \textbf{1.89} & \textbf{10.33} &
                 \textbf{28.87} & \textbf{9.67} & \textbf{53.43} \\
& & \texttt{Decomp} & 0.68 & 0.33 & 1.79 &
                 8.66 & 4.07 & 22.98 &
                 72.89 & 25.91 & 170.25 \\
& & \texttt{Galaxy} & 2.16 & 0.96 & 10.61 &
                 9.72 & 3.41 & 17.08 &
                 98.96 & 21.46 & 134.47 \\
\cmidrule(lr){2-12}
& \multirow{3}{*}{No. of facets/constraints}
& \texttt{RISP} & 6 & 1 & 11 &
                 8 & 2 & 14 &
                 7 & 2 & 14 \\
& & \texttt{Decomp} & 12 & 1 & 16 &
                 12 & 1 & 17 &
                 13 & 1 & 17 \\
& & \texttt{Galaxy} & \textbf{5} & \textbf{2} & \textbf{14} &
                 \textbf{5} & \textbf{1} & \textbf{10} &
                 \textbf{4} & \textbf{1} & \textbf{10} \\
\cmidrule(lr){2-12}
& \multirow{3}{*}{Area $(\mathrm{m}^2)$}
& \texttt{RISP} & 297.08 & 216.14 & 1393.80 &
                 29.57 & 26.59 & 210.58 &
                 1.23 & 1.46 & 13.29 \\
& & \texttt{Decomp} & \textbf{316.00} & \textbf{224.94} & \textbf{1407.20} &
                 32.64 & 28.53 & 193.62 &
                 1.42 & 1.66 & 15.17 \\
& & \texttt{Galaxy} & 243.46 & 181.74 & 1016.14 &
                 \textbf{33.34} & \textbf{28.05} & \textbf{225.54} &
                 \textbf{1.59} & \textbf{1.76} & \textbf{16.56} \\
\midrule
\multirow{12}{*}{3-D}
& \multirow{4}{*}{Time (ms)}
& \texttt{RISP} & \textbf{0.55} & \textbf{0.22} & \textbf{1.33} &
                 \textbf{5.30} & \textbf{2.81} & \textbf{14.84} &
                 \textbf{47.74} & \textbf{13.90} & \textbf{110.48} \\
& & \texttt{Decomp} & 2.74 & 1.61 & 9.44 &
                 24.58 & 14.03 & 78.64 &
                 203.76 & 87.90 & 513.62 \\
& & \ac{FIRI} & 6.75 & 2.16 & 13.22 &
                 36.37 & 14.05 & 73.61 &
                 448.18 & 92.95 & 572.94 \\
& & \texttt{CIRI} & 1.68 & 0.64 & 3.96 &
                 18.12 & 8.37 & 40.20 &
                 202.44 & 46.74 & 283.96 \\
\cmidrule(lr){2-12}
& \multirow{4}{*}{No. of facets/constraints}
& \texttt{RISP} & 13 & 4 & 27 &
                 20 & 5 & 39 &
                 22 & 5 & 38 \\
& & \texttt{Decomp} & 18 & 3 & 30 &
                 21 & 3 & 29 &
                 22 & 2 & 32 \\
& & \ac{FIRI} & 22 & 4 & 34 &
                 24 & 4 & 36 &
                 23 & 3 & 34 \\
& & \texttt{CIRI} & \textbf{13} & \textbf{3} & \textbf{24} &
                 \textbf{15} & \textbf{3} & \textbf{26} &
                 \textbf{17} & \textbf{3} & \textbf{24} \\
\cmidrule(lr){2-12}
& \multirow{4}{*}{Volume $(\mathrm{m}^3)$}
& \texttt{RISP} & 1375.92 & 916.81 & 4955.34 &
                 144.75 & 122.11 & 811.64 &
                 5.51 & 5.73 & 55.11 \\
& & \texttt{Decomp} & 1451.44 & 937.16 & 5062.65 &
                 163.21 & 128.71 & 892.71 &
                 8.20 & 6.82 & 61.51 \\
& & \ac{FIRI} & \textbf{1968.03} & \textbf{1096.88} & \textbf{6084.02} &
                 \textbf{246.95} & \textbf{183.57} & \textbf{1295.56} &
                 \textbf{25.74} & \textbf{35.92} & \textbf{337.18} \\
& & \texttt{CIRI} & 1449.65 & 924.09 & 5245.47 &
                 176.97 & 126.28 & 779.84 &
                 12.49 & 7.85 & 71.77 \\
\bottomrule
\end{tabular}
\end{table*}

The path-level results reported in the previous section demonstrate end-to-end performance, combining single-polytope construction with path-dependent operations, such as corridor chaining, waypoint membership tests, and intersection queries. To isolate the per-polytope scaling behavior of \texttt{RISP}, we perform a synthetic test in which we grow a \emph{single} polytope around one seed sampled from a  $50\times50\times5$~m workspace for the 3-D case and from $50\times50$~m workspace for the 2-D case. In this test, environments were divided into three point-cloud-density regimes: \emph{sparse} ($1{\times}10^4$--$5{\times}10^4$ points), \emph{medium} ($1{\times}10^5$--$5{\times}10^5$ points), and \emph{dense} ($1{\times}10^6$--$5{\times}10^6$ points). Each point cloud is generated artificially by randomly sampling points in the workspace\footnote{https://github.com/HKUST-Aerial-Robotics/mockamap}, creating 500 3-D and 500 2-D environments for each density. A total of $3000$ unique environments are generated using this procedure. 
For each 3-D environment, we select a seed at a fixed height $z=2.5$~m, with its $x,y$ coordinates drawn uniformly at random. For each 2-D environment, we neglect the $z$-coordinate and sample only the $x,y$ coordinates. In both cases, any seed that coincides with a point in the point cloud is resampled.

\texttt{RISP} achieves the lowest mean generation time in all six density/dimension settings of Table~\ref{tab:performance_metrics}, doing so without increasing polytope complexity as measured by the number of facets. The difference is most evident in the dense environments: in 3-D, \texttt{RISP} requires $47.74$~ms on average, corresponding to computational times that are $4.24\times$, $4.27\times$, and $9.39\times$ lower than those of \texttt{CIRI}, \texttt{Decomp}, and \ac{FIRI}, respectively. In 2-D, \texttt{RISP} is $2.52\times$ faster than \texttt{Decomp} and $3.43\times$ faster than \texttt{Galaxy}. The advantage is further demonstrated by the fact that, in all three 3-D settings and in the dense 2-D setting, the maximum runtime of \texttt{RISP} remains below the mean runtime of every corresponding baseline. At the same time, its per-polytope facet counts remain moderate: in 3-D, \texttt{RISP} averages $13$--$22$ facets, while in 2-D it uses only $6$--$8$, substantially fewer than \texttt{Decomp}'s $12$--$13$ and only modestly more than \texttt{Galaxy}'s $4$--$5$. Thus, the reduction in per-polytope construction time is not achieved at the expense of more complex returned polytopes, as measured by their facet and constraint counts. Note that these counts characterize individual regions and should not be confused with the corridor-level polytope and constraint counts of Table~\ref{tab:performance_metric}.

The computational gain in 3-D is accompanied by more conservative region size. \ac{FIRI} produces the largest regions, with mean volumes $1.43$--$4.68\times$ those of \texttt{RISP}, while the differences to \texttt{CIRI} and \texttt{Decomp} are smaller. The dense 3-D case illustrates this tradeoff most clearly: relative to \ac{FIRI}, \texttt{RISP} returns a $4.68\times$ smaller mean region while requiring $9.39\times$ less computation time. In 2-D, however, this loss in coverage is much less apparent: \texttt{RISP}, \texttt{Decomp}, and \texttt{Galaxy} produce broadly comparable mean areas across all three density levels. Relative to \texttt{RISP}, \texttt{Decomp} yields only $1.06$--$1.15\times$ larger regions, while \texttt{Galaxy} ranges from $0.82$--$1.29\times$ its area. Thus, particularly in 2-D, the reduction in computation time is achieved while retaining similar per-polytope coverage. Overall, the results characterize a favorable computation--coverage tradeoff over the tested environments: \texttt{RISP} prioritizes fast construction and moderate geometric complexity without requiring a large reduction in region size.

\section{Application to Autonomous Navigation}
\label{sec:sim_results}

We validate the complete framework on two robots, a quadrotor and a quadruped, navigating \emph{a~priori} unknown, cluttered environments using only onboard range sensing. In both cases, the robot carries a LiDAR (Velodyne VLP-16 which supplies point clouds at $15$--$20$ Hz) and replans a corridor at the sensor rate: as it moves, the local point cloud $\mathcal{P}_k$ and a global reference path $\Pi_k$ (computed with \ac{JPS}~\cite{Harabor2011jps}, implementation in~\cite{Liu2017RAL}) are refreshed with every incoming scan, and \texttt{PathCover} converts the current view into a safe corridor ahead of the robot. To validate the framework, we consider two robotic systems with different planners. The first is a simulated quadrotor and the second a physical implementation on the Ghost Robotics Vision60 quadruped and serves as a hardware validation under realistic LiDAR noise. In both cases, the corridor is capped at a predefined planning horizon for computational efficiency,\footnote{Capping the corridor does not affect the completeness result established in Theorem~\ref{thm:pathcover_completeness}, since the theorem continues to apply to the considered sub-path. In practice, the horizon is limited to avoid processing the entire reference path and unnecessarily increasing computation time.} and the goal is considered reached when the robot enters a sphere of radius $0.2\,\mathrm{m}$ centered at the goal. To account for the robot's finite dimensions, we inflate the point cloud by the robot's radius.

\subsection{Quadrotor: Flatness-Based Trajectory Optimization}
\label{subsec:sim_uav}

We deploy \texttt{PathCover} in a receding-horizon Gazebo simulation of a quadrotor in a $40\,\mathrm{m}\times20\,\mathrm{m}\times2\,\mathrm{m}$ office-like environment (Fig.~\ref{fig:sim_drone}). The horizon is tuned at six polytopes based on the chosen environment and trajectory optimization time configuration. At each cycle $k$, the accumulated point cloud and reference path are refreshed, and the algorithm is set to return at most six polytopes (i.e., $\mathcal{G}_k=\{\mathcal{H}_{k,1},\ldots,\mathcal{H}_{k,J_k}\}$ with $J_k\leq6$) together with a local goal $\mathrm{y}_\mathrm{int}$ (see Fig.~\ref{fig:sim_drone}). The planner then optimizes a trajectory within $\mathcal{G}_k$, of which only the initial segment is tracked before the next scan prompts the construction of a new corridor and a subsequent replanning step. The corridor slides forward with the robot at each scan, reflecting the receding-horizon operation of the planning loop.

The quadrotor planner uses \ac{COM} position and yaw as flat outputs~\cite{Mellinger2011MinimumSnap} and tracks the resulting trajectory with a geometric controller~\cite{Lee2010CDC}. Trajectory optimization inside the corridor $\mathcal{G}_k$ is performed using \texttt{GCOPTER}~\cite{Gao2022Gcopter}, a geometrically constrained trajectory-optimization framework for multicopters that exploits differential flatness and the minimum-control-effort (MINCO) trajectory representation to efficiently optimize smooth trajectories subject to spatial/corridor and kinodynamic constraints. In particular, \texttt{GCOPTER} accepts convex corridor regions represented as in \eqref{eq:poly-def} and constrains the corresponding trajectory portions to remain within these regions. Polytopic constraints are handled through the geometrical constraint-elimination machinery of \texttt{GCOPTER}, allowing the resulting problem to be solved efficiently as an unconstrained nonlinear optimization~\cite{Gao2022Gcopter} using  quasi-Newton methods.

Consequently, both the number of corridor polytopes and their geometric complexity influence the downstream optimization problem. In particular, the number of facets determines the size of the halfspace representation $A_{k,j}\mathrm{y}\leq b_{k,j}$ supplied to the trajectory optimizer and contributes to the geometric processing required for the corridor constraints. Thus, the substantially smaller number of polytopes and facets produced by \texttt{PathCover} provides a more compact corridor representation for \texttt{GCOPTER}. This connection is also reflected in the benchmark results of Table~\ref{tab:performance_metric}, where the reduced facet counts obtained with \texttt{PathCover} are accompanied by substantially lower downstream trajectory-optimization times.

Replanning is triggered by the simulated LiDAR scans at up to $20$ Hz, and the geometric controller runs at $1$ kHz. In the recorded run, computing $\Pi_k$ and generating $\mathcal{G}_k$ took $12.05$ ms on average (standard deviation $6.97$ ms; largest observed cycle $43.23$ ms) as the accumulated point cloud grew from $6.62{\times}10^4$ to $1.99{\times}10^5$ points. Figure~\ref{fig:sim_drone} demonstrates that the resulting point-cloud-separated regions can be directly supplied to \texttt{GCOPTER} to constrain the simulated trajectory. 

\begin{figure*}[t!]
\centering
\subfigure[]{\includegraphics[width=0.53\textwidth]{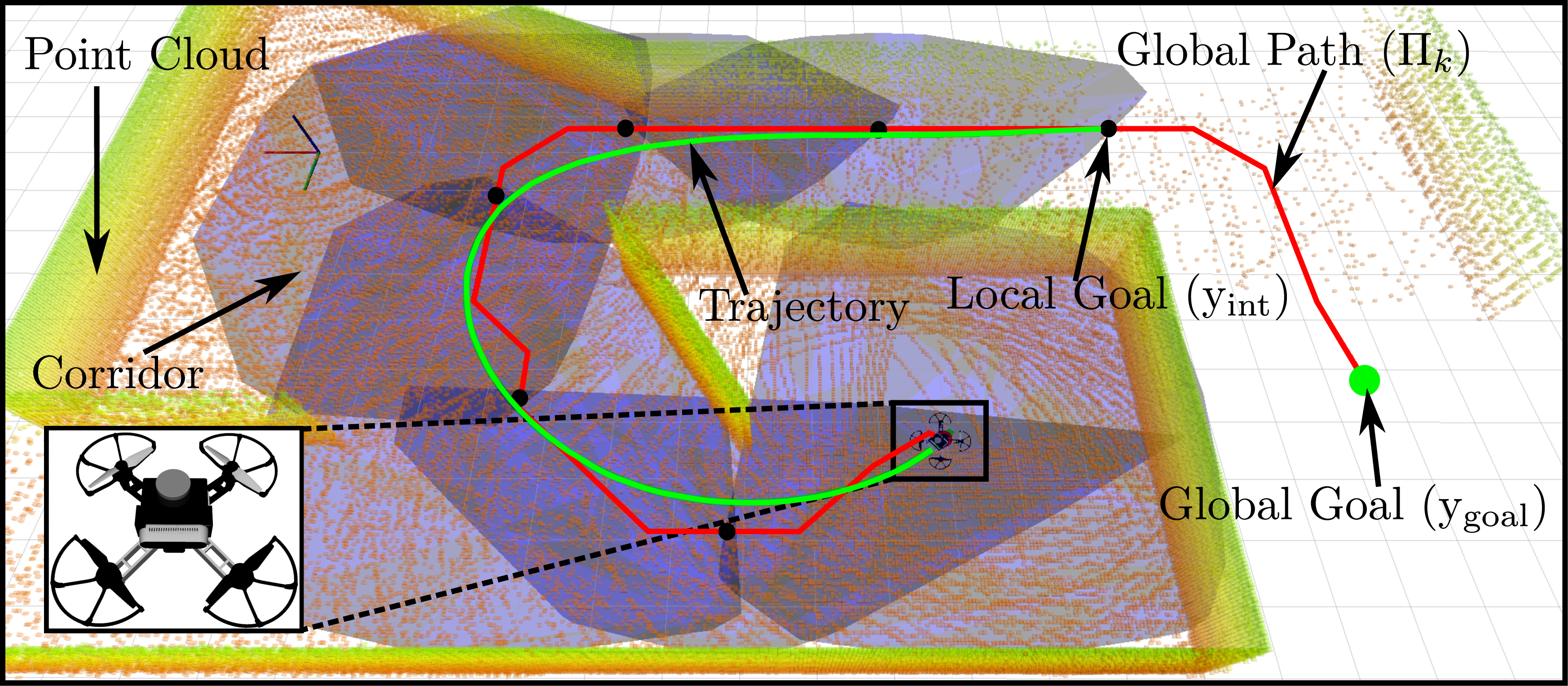}
\label{fig:sim_drone}}
\centering
\subfigure[]{\includegraphics[width=0.44\textwidth]{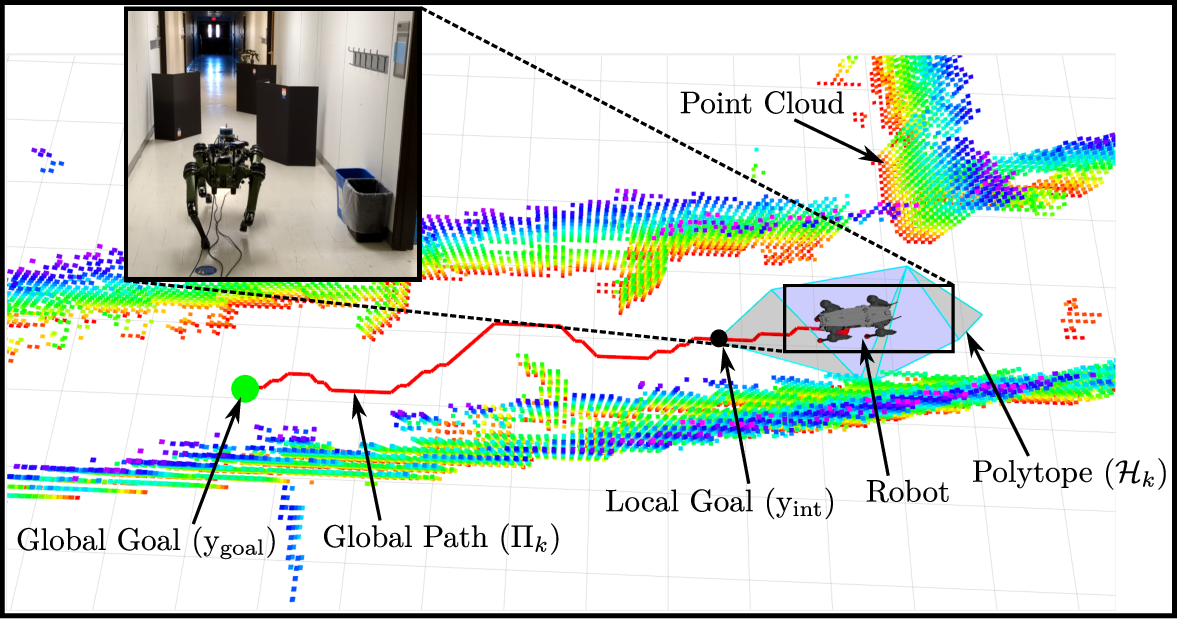}
\label{fig:exp_quadruped}}
\caption{Receding-horizon corridor planning. (a)~Quadrotor in Gazebo/RViz
tracking a smooth \texttt{GCOPTER} trajectory toward the local target
$\mathrm{y}_\mathrm{int}$, confined inside the multi-polytope point-cloud-separated corridor
$\mathcal{G}_k$ (blue, here $J_k=6$) generated along the global path (red)
toward $\mathrm{y}_\mathrm{goal}$ (green). (b)~Vision60 quadruped running the
single-polytope unicycle \ac{MPC}~\eqref{eq:mpc_unified} ($J_k=1$), tracking
$\mathrm{y}_\mathrm{int}$ (black) at the intersection of $\mathcal{H}_k$ and
$\Pi_k$ while advancing toward $\mathrm{y}_\mathrm{goal}$.}
\label{fig:drone_corridor}
\end{figure*}

\subsection{Quadruped: Unicycle-Based MPC Navigation}
\label{subsec:exp_quadruped}

We next validate the framework in hardware, on the Vision60 quadruped equipped with Velodyne VLP-16 LiDAR traversing a narrow indoor corridor formed by two parallel walls of box-shaped obstacles. The robot's proprietary controller handles low-level joint control, while our planner issues higher-level body-velocity commands, giving a hardware validation under realistic LiDAR noise. Given the quadruped's modest speed and the short prediction horizon, a single look-ahead polytope suffices: \texttt{PathCover} is invoked with a corridor horizon $J_k=1$ at each cycle $k$, returning one polytope $\mathcal{H}_k=\{\mathrm{y} \in \mathbb{R}^2\mid A_k\mathrm{y}\leq b_k\}$ per cycle. Unlike the quadrotor example, we retain the unicycle dynamics inside the optimization and solve a receding-horizon \ac{MPC} against this polytope.

We model the quadruped's body motion with the unicycle kinematics. The state stacks the \ac{COM} planar position and torso yaw, $\chi_k=\begin{bmatrix}x_k & y_k & \phi_k\end{bmatrix}^\top$, and evolves under the Euler-discretized model
\begin{equation}\label{eq:unicycle}
\chi_{k+1} = f(\chi_k, u_k) = \chi_k+\Delta t
\begin{bmatrix} v_k\cos\phi_k & v_k\sin\phi_k & \dot{\phi}_k \end{bmatrix}^\top,
\end{equation}
with body-velocity input $u_k=\begin{bmatrix}v_k & \dot{\phi}_k\end{bmatrix}^\top \in\mathcal{U}$, step $\Delta t=0.05\,\mathrm{s}$, and $\mathcal{U}$ a compact polytopic set encoding the saturation bounds. The corridor constraint acts on the position $\mathrm{y}_k=\begin{bmatrix}x_k & y_k\end{bmatrix}^\top$ directly.
At each cycle the robot is driven toward the desired state
$\chi_\mathrm{des}=\begin{bmatrix}\mathrm{y}_\mathrm{int}^\top & \phi_\mathrm{des}\end{bmatrix}^\top$, where $\mathrm{y}_\mathrm{int}$ is the local goal computed as in Section~\ref{subsec:sim_uav} (see Fig.~\ref{fig:exp_quadruped}) and $\phi_\mathrm{des}$ is the heading of the vector connecting $\mathrm{y}_k$ to $\mathrm{y}_\mathrm{int}$~\cite{Narkh2022RAL}. With running cost
$\mathcal{J}_\kappa(\chi_\kappa,u_\kappa)=\|\chi_\kappa-\chi_\mathrm{des}\|^2_Q
+\|u_\kappa\|^2_R$ and terminal cost
$\mathcal{J}_{k+H}(\chi_{k+H})=\|\chi_{k+H}-\chi_\mathrm{des}\|^2_{Q_f}$
($Q,R,Q_f\succ 0$), the receding-horizon problem solved at each cycle is
\vskip -5pt
\noindent\rule{\columnwidth}{0.5pt}
\vskip -15pt
\begin{subequations}
\label{eq:mpc_unified}
\begin{IEEEeqnarray}{s'rCl}
\label{eq:mpc_unified_1}
$\underset{X,U}{\text{minimize}}$ & \IEEEeqnarraymulticol{3}{l}
{\displaystyle\sum_{\kappa=k}^{k+H-1} \mathcal{J}_\kappa(\chi_\kappa, u_\kappa)
+ \mathcal{J}_{k+H}(\chi_{k+H})} \\
\label{eq:mpc_unified_2}
\text{subject to}
& \IEEEeqnarraymulticol{3}{l}{\chi_{\kappa+1} = f(\chi_\kappa, u_\kappa),
~~ \chi_k = \chi_\mathrm{init},} \\
\label{eq:mpc_unified_3}
& \mathrm{y}_\kappa &\in& \mathcal{H}_k, ~~ \mathrm{y}_{k+H} \in \mathcal{H}_k, ~~ u_\kappa \in \mathcal{U},
\end{IEEEeqnarray}
\end{subequations}
\vskip -10pt
\noindent\rule{\columnwidth}{0.5pt}
where $X=\{\chi_{k+1},\dots,\chi_{k+H}\}$ and $U=\{u_k,\dots,u_{k+H-1}\}$. The polytope $\mathcal{H}_k$ is held constant over the time horizon $H$ and refreshed only between cycles, at the LiDAR rate. We use a time horizon of $H=20$ steps and solve \eqref{eq:mpc_unified} with FATROP~\cite{fatrop}.

Fig.~\ref{fig:exp_quadruped} shows the Vision60 quadruped clearing the corridor: the robot tracks $\mathrm{y}_\mathrm{int}$ while $\mathcal{H}_k$ is regenerated in real time from each incoming scan, its boundary reshaping as the robot advances. Over the experiment, planning $\Pi_k$ and generating $\mathcal{H}_k$ took $2.93\,\mathrm{ms}$ on average (with standard deviation $0.87\,\mathrm{ms}$), and the slowest cycle completed in $7.08\,\mathrm{ms}$ on point clouds ranging from $2.6{\times}10^4$ to $1.5{\times}10^5$ points. Beyond timing considerations, the experiment supports the central claim of this section: the framework transfers from simulation to hardware without requiring any algorithmic modifications. The same \texttt{PathCover} pipeline, differing only in the downstream planner it feeds, enables continuous closed-loop navigation in unknown environments under noisy sensor data.

\subsection{Elimination Dynamics of \texttt{RISP}}
\label{subsec:empirical_support}

Having validated the framework end-to-end, we revisit the theoretical analysis of \texttt{RISP} and examine whether its complexity guarantees are reflected in practice. Theorem~\ref{thm:RISP_complexity} establishes $O(n)$ expected time complexity under condition \eqref{eq:beta-good}, which requires that each iteration eliminates at least some constant fraction of the remaining points with nonzero probability. While this condition is difficult to verify analytically for arbitrary point clouds, we evaluate whether it is consistent with the empirical behavior of \texttt{RISP} across all experimental datasets. For each \texttt{RISP} call, we record the full sequence $\{N_k\}$ of remaining point counts across iterations, and compute the per-step fractional elimination $\eta_k = (N_k - N_{k+1})/N_k$. Under Theorem~\ref{thm:RISP_complexity}, a step eliminates at least $\beta$ fraction of points if $\eta_k \geq \beta$; accordingly, we define the empirical $\beta$ elimination probability as $\hat{r}(\beta) = \Pr \{\eta_k \geq \beta\}$, estimated across all steps and runs in each dataset. The effective decay rate of the geometric envelope is $\beta \hat{r}(\beta)$, which trades off the elimination threshold $\beta$ against how frequently it is achieved. We estimate the tightest such rate by computing
\begin{equation}\nonumber
\beta^*\, \;=\; \operatorname*{arg\,max}_{\beta \,\in\, (0,1)}
  \;\beta \cdot \hat{r}(\beta) ~~\text{and}~~ \hat{r}^* = \hat{r}(\beta^*)
\end{equation}
over a uniform grid of $\beta$. This directly identifies the pair $(\beta^*, \hat{r}^*)$ consistent with Theorem~\ref{thm:RISP_complexity} that yields the fastest predicted convergence. The empirical mean $\mathbb{E}[N_k]$ is then computed by averaging the sequences of $N_k$ across all runs and compared against the geometric envelope $\overline{n}(1-\beta^*\hat{r}^*)^k$ from \eqref{eq:geometric_decay} in the proof of Theorem~\ref{thm:RISP_complexity} presented in Appendix~\ref{app:theorem2}, where $\overline{n}$ is the mean initial point count.

Fig.~\ref{fig:complexity_decay} reports the empirical $\mathbb{E}[N_k]$ for both scenarios, averaged over $1373$ \texttt{RISP} calls for the Gazebo quadrotor simulation (mean initial count $\overline{n}\approx 1.97{\times}10^5$ points) and $991$ calls for the real-world quadrupedal experiments ($\overline{n}\approx 1.12{\times}10^5$ points). In both cases $\mathbb{E}[N_k]$ decays geometrically and remains below the fitted envelope $\overline{n}(1-\beta^*\hat{r}^*)^k$, with optimal effective decay rates $\beta^*\hat{r}^* = 0.056$ (quadrotor) and $\beta^*\hat{r}^* = 0.087$ (quadruped). The consistency of these estimates across datasets with different sensor characteristics, obstacle geometries, and initial point counts confirms that real LiDAR point clouds exhibit the multiplicative drift required by Theorem~\ref{thm:RISP_complexity} in a distributional sense. The key quantitative check is whether termination occurs within the drift-theoretic bound. The predicted iteration bounds $(1+\ln\overline{n})/(\beta^*\hat{r}^*)$ evaluate to $236$ and $145$ iterations for the quadrotor and quadruped clouds, respectively, whereas the observed mean termination times are $66$ and $43$ iterations, comfortably inside those bounds, as the vertical dashed lines in Fig.~\ref{fig:complexity_decay} indicate. Together, these results explain the consistent $O(n)$ runtime and $O(\log n)$ iteration count of \texttt{RISP} observed throughout both simulated and real-world experiments, and provide strong evidence that Theorem~\ref{thm:RISP_complexity} is an intrinsic property of structured sensor data.

\begin{figure}[t!]
\centering
\includegraphics[width=0.8\columnwidth]{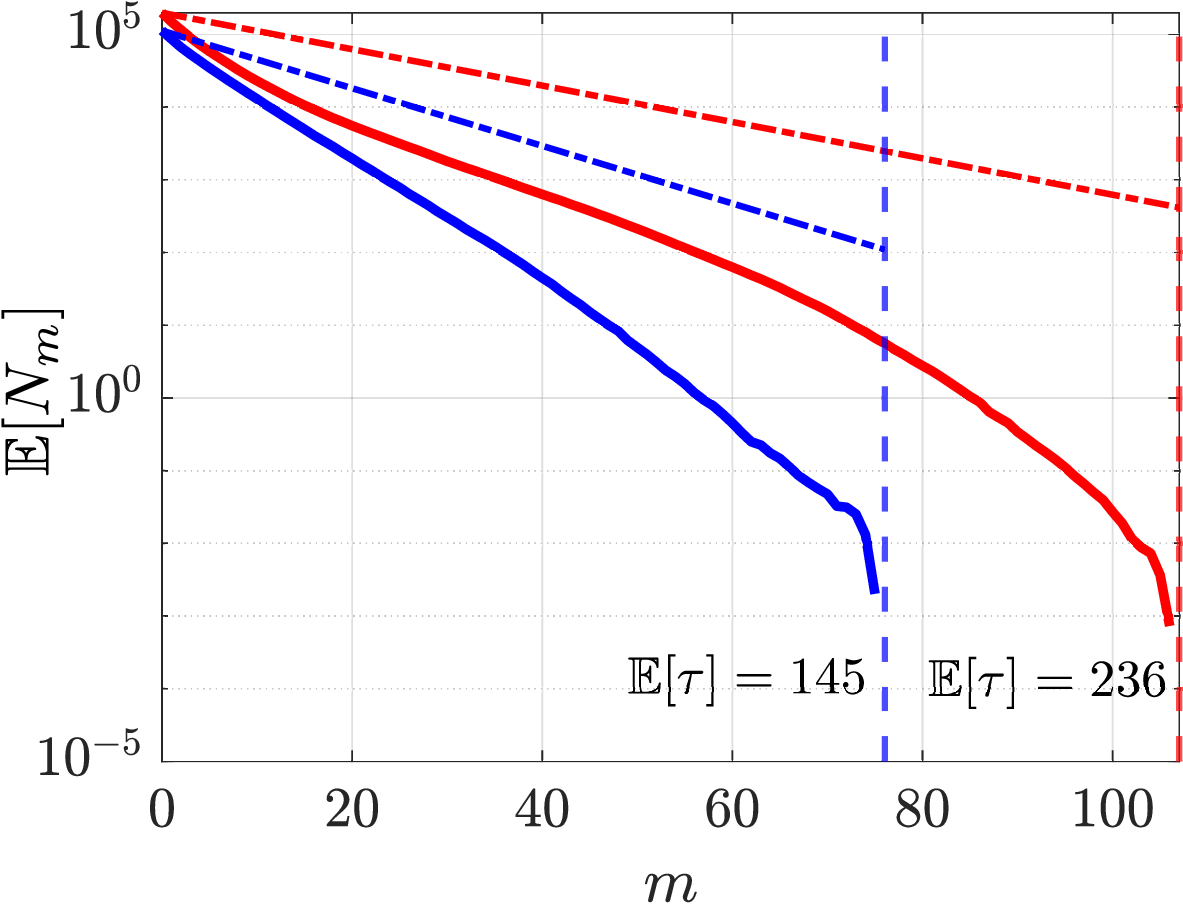}
\vskip -5pt
\caption{Empirical support for Theorem~\ref{thm:RISP_complexity}. For the Gazebo quadrotor simulation (Sec.~\ref{subsec:sim_uav}) (red) and the real-world quadrupedal experiments (Sec.~\ref{subsec:exp_quadruped}) (blue). Each solid curve is the empirical $\mathbb{E}[N_k]$ averaged over all \texttt{RISP} calls in the dataset, and the dashed curve is the corresponding fitted geometric envelope $\overline{n}(1-\beta^*\hat{r}^*)^k$ from \eqref{eq:geometric_decay}. In both cases $\mathbb{E}[N_k]$ decays geometrically and stays below its envelope, and termination (vertical dashed lines) occurs well within the predicted bound (see App.~\ref{app:theorem2}); consistent with the $O(\log n)$ expected iteration count predicted by the theorem.}
\label{fig:complexity_decay}
\vskip -15pt
\end{figure}

\section{Conclusion}

We presented \texttt{PathCover}, a path-coverage framework for generating safe corridors from point-cloud obstacle data for real-time autonomous navigation. The framework is built on the novel randomized polytope construction algorithm \ac{RISP} that separates a seed point from a finite obstacle point set and enables the sequential construction of overlapping convex polytopes along a reference path. We proved finite-time termination and path coverage under stated clearance assumptions, and established the expected linear-time complexity of \ac{RISP} under a probabilistic point-elimination condition. Benchmark evaluations show that \texttt{PathCover} achieves the lowest corridor-generation time while producing a substantially less complex corridor representation, both in terms of the number of polytopes and the number of facets per corridor, despite generating individual regions of smaller volume. Importantly, this compact polytope representation also reduces the computational burden of downstream trajectory optimization. Synthetic scalability tests further support the expected near-linear behavior of \ac{RISP} over a wide range of point cloud sizes. Finally, we validated the resulting corridors in closed-loop navigation through quadrotor simulation and physical quadrupedal deployment. These results indicate that randomized point cloud corridor construction is a promising approach for sensor-rate replanning on autonomous robotic systems, particularly when predictable computation time is prioritized over maximizing individual polytope volume.

\appendices

\section{Proof of Theorem 1}
\label{app:theorem1}

To prove Theorem~\ref{thm:pathcover_completeness}, we first state and prove two lemmas. 

\begin{lemma}[Properties of the \texttt{RISP} output]
\label{lem:RISP_polytope_properties}
Let $\mathcal{W} \subset \mathbb{R}^d$ and $\mathcal{O} \subset \mathcal{W}$ satisfy Assumption~\ref{ass:free_space} and define the free space $\mathcal{F} = \mathcal{W} \setminus \mathcal{O}$. Let $\mathrm{y}_{\mathrm{seed}}$ be a seed point in the interior of $\mathcal{F}$, and let $\mathcal{P}$ be a finite set of points sampled from the boundaries of the obstacles $\mathcal{O}$. Choose $\alpha \in (0,1)$. Then, the output $\mathcal{H}$ of the \texttt{RISP} algorithm satisfies the following:
\begin{enumerate}[label=(\roman*)]
    \item $\mathcal{H}$ is a polytope that contains $\mathrm{y}_{\mathrm{seed}}$ in its interior;
    \item $\mathcal{H}$ is strictly separated from every point $\mathrm{p} \in \mathcal{P}$.
\end{enumerate}
\end{lemma}

\begin{proof}
Fix $\alpha \in (0,1)$. 
We first prove part (i) by examining the process by which the \texttt{RISP} algorithm constructs polytopes. Set $\mathcal{P}_0=\mathcal{P}$. Given a seed point $\mathrm{y}_\mathrm{seed}$ in the interior of $\mathcal{F}$, execute line 4 of Algorithm~\ref{alg:risp} to select a point from $\mathcal{P}_0$ uniformly at random; let $\mathrm{p}_0$ be this point. Then, execute line 5 of Algorithm~\ref{alg:risp} to compute $\mathrm{a}_0 = \mathrm{p}_0 - \mathrm{y}_\mathrm{seed}$ and $\mathrm{b}_0 = \mathrm{a}^\top_0 \mathrm{y}_\mathrm{seed} + (1-\alpha) \| \mathrm{a}_0 \|^2_2$ by~\eqref{eq:plane_selection}. This defines the hyperplane $h_0$ generated by $\mathrm{p}_0$ according to~\eqref{eq:hyperplane}--\eqref{eq:plane_selection}; by construction, $h_0$ is orthogonal to $\mathrm{a}_0$ and passes through the point $\alpha \mathrm{y}_\mathrm{seed} + (1-\alpha) \mathrm{p}_0$. Now, let $\mathcal{B}_0(\mathrm{y}_\mathrm{seed})$ be the (open) ball of radius $\rho_0 = (1-\alpha) \, \|\mathrm{a}_0\|_2$ centered at the point $\mathrm{y}_\mathrm{seed}$, and define the affine function $f_0(\mathrm{x}) = \mathrm{a}_0^\top \mathrm{x} - \mathrm{b}_0$. Then, 
\begin{align}
    \nonumber
    f_0(\mathrm{x}) &= \mathrm{a}^\top_0 \mathrm{x} - \mathrm{b}_0 \\
    \nonumber
    &= \mathrm{a}^\top_0 (\mathrm{x} - \mathrm{y}_\mathrm{seed}) - (1-\alpha) \| \mathrm{a}_0 \|^2_2 \\
    \nonumber
    &\leq \| \mathrm{a}_0 \|_2 \cdot \| \mathrm{x} - \mathrm{y}_\mathrm{seed} \|_2 - (1-\alpha) \| \mathrm{a}_0 \|^2_2 \\
    \nonumber
    & = \| \mathrm{a}_0 \|_2 \left( \| \mathrm{x} - \mathrm{y}_\mathrm{seed} \|_2 -\rho_0 \right)
\end{align}
implying $f_0(\mathrm{x}) < 0$ for all $\mathrm{x} \in \mathcal{B}_0(\mathrm{y}_\mathrm{seed})$; note that the Cauchy-Schwarz inequality was used to pass from the second to the third line. As a result, the open ball $\mathcal{B}_0(\mathrm{y}_\mathrm{seed})$ is entirely contained in the closed halfspace $h^-_0 = \{\mathrm{x} \in \mathbb{R}^d ~|~ \mathrm{a}^\top_0 \mathrm{x} \leq \mathrm{b}_0 \}$ and so $\mathrm{y}_\mathrm{seed}$ is an interior point of $h^-_0$. Note also that $\mathrm{p}_0 \notin h^-_0$ since $f_0(\mathrm{p}_0) > 0$. Next, following lines 7--9 of Algorithm~\ref{alg:risp}, discard all points in $\mathcal{P}_0$ that lie in the side of $h_0$ opposite to $\mathrm{y}_{\mathrm{seed}}$; i.e., discard all points in $\mathcal{Q}_0 = \{ \mathrm{q} \in \mathcal{P}_0 ~|~ \mathrm{a}_0^{\top}\mathrm{q} > \mathrm{b}_0 \}$. Let $\mathcal{P}_1 = \mathcal{P}_0 \setminus \mathcal{Q}_0$ be the set of remaining points. 

Repeating the same procedure, at iteration $m \geq 1$, the set of remaining points is $\mathcal{P}_m = \mathcal{P}_{m-1} \setminus \mathcal{Q}_{m-1}$. A new point $\mathrm{p}_m \in \mathcal{P}_m$ is selected uniformly at random and used to generate the closed halfspace $h^-_m$, as above. This halfspace includes the open ball $\mathcal{B}_m(\mathrm{y}_\mathrm{seed})$ in its interior and excludes the corresponding generating point $\mathrm{p}_m$. Furthermore, its boundary, $h_m$, is used to discard all points in $\mathcal{P}_m$ that lie in the side of $h_m$ opposite to $\mathrm{y}_{\mathrm{seed}}$. After finitely many iterations, all points in the point cloud are exhausted, yielding a finite collection of closed halfspaces, $\{ h^-_m, ~m = 0,..., M \}$. The intersection of these halfspaces, $h_\star = \bigcap^M_{m=0} h^-_m$, is a non-empty polyhedron that includes the open ball $\mathcal{B}_\star(\mathrm{y}_\mathrm{seed}) = \bigcap^M_{m=0} \mathcal{B}_m(\mathrm{y}_\mathrm{seed})$ in its interior. 

Next, line 13 of Algorithm~\ref{alg:risp} augments the halfspace description of $h_\star$ with the constraints defining $\mathcal{W}$. Since, by Assumption~\ref{ass:free_space}, $\mathcal{W}$ is a polytope (i.e., a bounded polyhedron), $h_\star \cap \mathcal{W}$ is the bounded intersection of finitely many halfspaces, and is therefore itself a polytope $\mathcal{H}$; thus, the output $\mathcal{H} = h_\star \cap \mathcal{W}$ of the \texttt{RISP} algorithm is a polytope. To show that $\mathcal{H}$ contains $\mathrm{y}_{\mathrm{seed}}$ in its interior, note that $\mathrm{y}_\mathrm{seed}$ is an interior point of $\mathcal{F} = \mathcal{W} \setminus \mathcal{O}$. Hence, by possibly reducing the radius of $\mathcal{B}_\star(\mathrm{y}_\mathrm{seed}) \subset h_\star$ one obtains an open ball entirely included in $\mathcal{F}$, therefore in $\mathcal{W}$ and consequently in $\mathcal{H} = h_\star \cap \mathcal{W}$, thus completing the proof of part (i).

To show part (ii), note first that $h_\star$ does not include any of the points in $\mathcal{P}$. This can be seen by a contradiction argument. Suppose, \emph{ad absurdum}, that there is a point $\mathrm{p}^* \in \mathcal{P}$ with $\mathrm{p}^* \in h_\star$. This implies that
\begin{equation}\label{eq:contradict}
    \mathrm{a}_m^\top \mathrm{p}^* \leq b_m, \quad \text{for all} \quad m=0, \ldots, M.    
\end{equation}
By the construction of $h_\star$, however, $\mathrm{p}^*$ was either used to generate a hyperplane that forms a boundary of $h_\star$ or has been discarded. In either case, $\mathrm{a}_m^\top \mathrm{p}^* > b_m$ for at least one $m$, which contradicts \eqref{eq:contradict}, implying that $h_\star$ cannot contain any point in $\mathcal{P}$. Since, $\mathcal{H} = h_\star \cap \mathcal{W}$ is the polytope returned by \texttt{RISP}, no point in the point cloud $\mathcal{P}$ can be in $\mathcal{H}$; that is, $\mathcal{H} \cap \mathcal{P} = \emptyset$. To show that $\mathcal{H}$ is strictly separated from every point $\mathrm{p} \in \mathcal{P}$ and complete the proof of part (ii), follow the arguments in~\cite[Example~2.20]{Boyd_cvx}, noting that $\mathcal{H}$ is a closed convex set and $\mathrm{p} \notin \mathcal{H}$ for any $\mathrm{p} \in \mathcal{P}$.
\end{proof}

In the statement of the following lemma, we use the following notation 
\begin{equation}\nonumber
    [\mathrm{y}_1, \mathrm{y}_2] = \{ (1-t) \mathrm{y}_{1} + t \mathrm{y}_{2} \mid t\in[0,1] \}
\end{equation}
to denote the line segment connecting two points $\mathrm{y}_1$ and $\mathrm{y}_2$. 

\begin{lemma}[Uniform clearance bound]
\label{lem:path_clearance}
Let $\mathcal{W}$, $\mathcal{O}$, $\mathcal{F}$, $\mathcal{P}$ and $\alpha \in (0,1)$ be as in Lemma~\ref{lem:RISP_polytope_properties}. Let $\Pi = \{\mathrm{y}_0, \dots, \mathrm{y}_L\}$ be a sequence of waypoints satisfying Assumption~\ref{ass:free_path} and define $\Gamma(\Pi) = \bigcup_{\ell=1}^{L} [\mathrm{y}_{\ell-1}, \mathrm{y}_\ell]$ as the path formed by the line segments connecting consecutive waypoints. Then, for any seed point $\mathrm{y}_\mathrm{seed} \in \Gamma(\Pi)$, the polytope $\mathcal{H}$ returned by $\mathtt{RISP}$ contains the open ball $\mathcal{B}_{\delta_\mathrm{min}}(\mathrm{y}_\mathrm{seed}),~\delta_\mathrm{min} = (1-\alpha) \epsilon_\Gamma$, where
\begin{equation} \label{eq:path_clearance}
    \epsilon_\Gamma = \min \bigl\{\min_{\mathrm{y} \in \Gamma(\Pi)}\mathtt{dist}(\mathrm{y}, \mathcal{P}), \min_{\mathrm{y} \in \Gamma(\Pi)} \mathtt{dist}(\mathrm{y}, \partial\mathcal{W})\bigr\}
\end{equation}
denotes the clearance of $\Gamma(\Pi)$ from the point cloud $\mathcal{P}$ and the boundary $\partial \mathcal{W}$ of the workspace $\mathcal{W}$.
\end{lemma}

\begin{proof}
First, note that $\epsilon_\Gamma > 0$ is well defined. Indeed, since $\mathcal{P}$ is finite and $\partial \mathcal{W}$ closed, the distances in \eqref{eq:path_clearance} are given by $\mathtt{dist}(\mathrm{y}, \mathcal{P}) \!=\! \min_{\mathrm{p} \in \mathcal{P}} \| \mathrm{y}-\mathrm{p} \|_2$ and $\mathtt{dist}(\mathrm{y}, \partial \mathcal{W}) \!=\! \min_{a \in \partial \mathcal{W}} \| \mathrm{y}-a \|_2$, respectively. Moreover, by Assumption~\ref{ass:free_path}, each line segment $[\mathrm{y}_{\ell-1}, \mathrm{y}_\ell]$ lies in $\mathrm{int}(\mathcal{F})$, and hence $\Gamma(\Pi)$ is disjoint from both $\mathcal{P}$ and $\partial \mathcal{W}$. As a result, the distances in \eqref{eq:path_clearance} are strictly positive for all $\mathrm{y} \in \Gamma(\Pi)$. Finally, since the functions $\mathtt{dist}(\,\cdot\,,\mathcal{P})$ and $\mathtt{dist}(\,\cdot\,,\partial \mathcal{W})$ are continuous, and each segment $[\mathrm{y}_{\ell-1}, \mathrm{y}_\ell]$, and hence the path $\Gamma(\Pi)$, is compact, it follows that $\min_{\mathrm{y} \in \Gamma(\Pi)}\mathtt{dist}(\mathrm{y}, \mathcal{P})$ and $\min_{\mathrm{y} \in \Gamma(\Pi)}\mathtt{dist}(\mathrm{y}, \partial \mathcal{W})$ are attained, and hence $\epsilon_\Gamma$ is well defined and strictly positive.   

Next, consider a seed point $\mathrm{y}_{\mathrm{seed}} \in \Gamma({\Pi})$. By 
Lemma~\ref{lem:RISP_polytope_properties}, each halfspace 
$h_m^{-} = \{\mathrm{x} \in \mathbb{R}^d ~|~ \mathrm{a}_m^\top \mathrm{x} \le b_m\}$ generated by \texttt{RISP} from a sampled obstacle point $\mathrm{p}_m \in \mathcal{P}$ contains the open ball $B_{\rho_m}(\mathrm{y}_{\mathrm{seed}})$ of radius $\rho_m = (1-\alpha)\mathrm{a}_m$ where $\mathrm{a}_m = \| \mathrm{p}_m - \mathrm{y}_{\mathrm{seed}}\|_2$. Clearly,
\begin{align}
    \nonumber
    \epsilon_\Gamma \leq \min_{\mathrm{y} \in \Gamma(\Pi)} \mathtt{dist} (\mathrm{y}, \mathcal{P})
    &\leq \mathtt{dist} (\mathrm{y}_\mathrm{seed}, \mathcal{P}) \\
    \nonumber
    & = \min_{\mathrm{p} \in \mathcal{P}} \|\mathrm{y}_\mathrm{seed} - \mathrm{p}_m \|_2 \\
    \nonumber
    & \leq \|\mathrm{y}_\mathrm{seed} - \mathrm{p}_m \|_2
\end{align}
for any $\mathrm{p}_m \in \mathcal{P}$. Multiplying both sides by $1-\alpha > 0$ yields $\delta_{\min} = (1-\alpha)\epsilon_\Gamma \leq \rho_m$ for any sampled point $\mathrm{p}_m \in \mathcal{P}$. Hence, all the halfspaces generated by \texttt{RISP} by sampling points $\mathrm{p}_m \in \mathcal{P}$ contain the ball $\mathcal{B}_{\delta_\mathrm{min}}(\mathrm{y}_\mathrm{seed})$. Taking the intersection over all these halfspaces gives $B_{\delta_{\min}}(\mathrm{y}_{\mathrm{seed}}) \subset \bigcap_{m} h_m^{-} = h_\star$. 
Similarly, since $1-\alpha > 0$, we have $\delta_\mathrm{min} = (1-\alpha) \epsilon_\Gamma < \epsilon_\Gamma \leq \mathtt{dist}(\mathrm{y}_\mathrm{seed}, \partial\mathcal{W})$, implying $\mathcal{B}_{\delta_\mathrm{min}}(\mathrm{y}_\mathrm{seed}) \subset \mathcal{W}$.
Combining these results, i.e., $B_{\delta_{\min}}(\mathrm{y}_{\mathrm{seed}}) \subset h_\star$ and $\mathcal{B}_{\delta_\mathrm{min}}(\mathrm{y}_\mathrm{seed}) \subset \mathcal{W}$, with the fact that $\mathcal{H} = h_\star \cap \mathcal{W}$, as defined in the proof of Lemma~\ref{lem:RISP_polytope_properties}, results in $\mathcal{B}_{\delta_\mathrm{min}}(\mathrm{y}_\mathrm{seed}) \subset  \mathcal{H}\enspace$.
\end{proof}

We are now ready to proceed with a proof of Theorem~\ref{thm:pathcover_completeness}.

\begin{proof}[Proof of Theorem~\ref{thm:pathcover_completeness}]
Our proof begins by constructing $\mathcal{H}_1$, and then proceeds inductively to generate a sequence of polytopes $\mathcal{H}_j$, $j \geq 2$, via repeated application of Lemma~\ref{lem:RISP_polytope_properties} using seed points as specified by Algorithm~\ref{alg:PathCover}. The proof ends by showing that $J < \infty$, which implies that Algorithm~\ref{alg:PathCover} terminates after finitely many steps and produces a finite sequence of polytopes $\mathcal{G} = \{\mathcal{H}_1, \dots, \mathcal{H}_J\}$.

To construct $\mathcal{H}_1$, \texttt{RISP} is executed with seed point $\mathrm{y}^{(1)}_\mathrm{seed} = \mathrm{y}_0$, the first waypoint in $\Pi$, which coincides with the robot's initial position. By Lemma~\ref{lem:RISP_polytope_properties}, the polytope $\mathcal{H}_1$ returned by \texttt{RISP} contains $\mathrm{y}^{(1)}_\mathrm{seed}$ in its interior and is strictly separated from every point in $\mathcal{P}$, thus satisfying condition~\ref{PCC:1}. 

For $j \geq 2$ the result follows by induction. We first consider the base case $j=2$. To construct $\mathcal{H}_2$, the algorithm proceeds sequentially through the waypoints in $\Pi$ until it encounters the first waypoint that does \emph{not} lie in $\mathcal{H}_1$; denote this waypoint by $\mathrm{y}_{\ell_1}$. By construction, all preceding waypoints $\mathrm{y}_\ell \in \Pi$ for $\ell=1,...,\ell_1-1$ are contained in $\mathcal{H}_1$; i.e., are covered by $\mathcal{H}_1$. Since $\mathrm{y}_{\ell_1-1} \in \mathcal{H}_1$, $\mathrm{y}_{\ell_1} \notin \mathcal{H}_1$, and $\mathcal{H}_1$ is convex, the line segment connecting $\mathrm{y}_{\ell_1-1}$ and $\mathrm{y}_{\ell_1}$ intersects the boundary of $\mathcal{H}_1$ at a unique point. This point is used as the new seed for \texttt{RISP}, denoted by $\mathrm{y}^{(2)}_\mathrm{seed}$. Then, by Lemma~\ref{lem:RISP_polytope_properties}, \texttt{RISP} returns a new polytope $\mathcal{H}_2$ that contains $\mathrm{y}^{(2)}_\mathrm{seed}$ in its interior and is strictly separated from all points in $\mathcal{P}$; hence, $\mathcal{H}_2$ satisfies  condition~\ref{PCC:1}. 
Furthermore, since $\mathrm{y}^{(2)}_\mathrm{seed}$ is in the line segment connecting two consecutive waypoints $\mathrm{y}_{\ell_1-1}$ and $\mathrm{y}_{\ell_1}$, Lemma~\ref{lem:path_clearance} implies that $\mathcal{H}_2$ contains the open ball $\mathcal{B}_{\delta_{\min}}(\mathrm{y}^{(2)}_\mathrm{seed})$, $\delta_{\min}>0$. Since the center $\mathrm{y}^{(2)}_\mathrm{seed}$ of this ball lies on the boundary of $\mathcal{H}_1$, and $\mathcal{B}_{\delta_{\min}}(\mathrm{y}^{(2)}_\mathrm{seed}) \subset \mathcal{H}_2$ it follows that $\mathcal{H}_1 \cap \mathcal{H}_2 \neq \emptyset$ and so $\mathcal{H}_1$ and $\mathcal{H}_2$ satisfy condition~\ref{PCC:2}  

For the inductive step, assume that, for some $j>2$, the polytopes $\mathcal{H}_2, \ldots, \mathcal{H}_j$ have been constructed and satisfy conditions \ref{PCC:1}-\ref{PCC:2}. Assume that all waypoints up to some index $\ell_j-1$ are covered by $\mathcal{H}_1$ and the collection $\mathcal{H}_i$, $i = 2, \ldots, j$, i.e., $\mathrm{y}_\ell \in \bigcup^j_{i=1} \mathcal{H}_i$ for $\ell = 0, \ldots, \ell_j-1$, and let $\mathrm{y}_{\ell_j}$ be the first waypoint in $\Pi$ that does \emph{not} lie in $\bigcup^j_{i=1} \mathcal{H}_i$. Then, exactly as in the base case above, the new seed point $\mathrm{y}^{(j+1)}_\mathrm{seed}$ used in \texttt{RISP} is defined as the unique intersection point between the boundary of $\mathcal{H}_j$ and the line segment joining $\mathrm{y}_{\ell_j-1}$ and $\mathrm{y}_{\ell_j}$. By Lemma~\ref{lem:RISP_polytope_properties}, \texttt{RISP} returns a new polytope $\mathcal{H}_{j+1}$ that contains $\mathrm{y}^{(j+1)}_\mathrm{seed}$ in its interior and satisfies \ref{PCC:1}. Furthermore, since $\mathrm{y}^{(j+1)}_\mathrm{seed}$ lies in the line segment connecting consecutive waypoints, Lemma~\ref{lem:path_clearance} implies that $\mathcal{H}_{j+1}$ contains the ball $\mathcal{B}_{\delta_{\min}}(\mathrm{y}^{(j+1)}_\mathrm{seed})$; since $\mathrm{y}^{(j+1)}_\mathrm{seed}$ lies on the boundary of $\mathcal{H}_j$, we have $\mathcal{H}_j \cap \mathcal{H}_{j+1} \neq \emptyset$ and \ref{PCC:2} is satisfied. Thus, the sequence of polytopes can be extended to include $\mathcal{H}_{j+1}$, completing the inductive step.

It remains to show that $J$ is finite. The proof relies on the observation that for any collection $\{\mathcal{H}_1,...,\mathcal{H}_j\}$ with $1 \leq j < J$, constructed as above, the subsequent polytope $\mathcal{H}_{j+1}$ guarantees progress along the path $\Gamma(\Pi)$. To see this, note that the seed point $\mathrm{y}^{(j+1)}_\mathrm{seed}$ used to construct $\mathcal{H}_{j+1}$ lies in the line segment connecting the last covered waypoint $\mathrm{y}_{\ell_j-1} \in \cup^j_{i=1} \mathcal{H}_i$ with the first waypoint $\mathrm{y}_{\ell_j} \notin \cup^j_{i=1} \mathcal{H}_i$ that is not yet covered. Let $\mathbf{v}_{\ell_{j+1}}$ denote the unit vector pointing from $\mathrm{y}_\mathrm{seed}^{(j+1)}$ toward $\mathrm{y}_{\ell_j}$ along the line segment connecting these two points. Since by Lemma~\ref{lem:path_clearance}, the polytope $\mathcal{H}_{j+1}$ contains $\mathcal{B}_{\delta_{\min}}(\mathrm{y}_\mathrm{seed}^{(j+1)})$, the point $\mathrm{y}_\mathrm{seed}^{(j+1)} + \delta_{\min} \mathbf{v}_{\ell_{j+1}}$ lies inside $\mathcal{H}_{j+1}$, demonstrating that $\mathcal{H}_{j+1}$ extends at least distance $\delta_{\min}>0$ toward $\mathrm{y}_{\ell_j}$. Consequently, each polytope extends at least $\delta_{\min} > 0$ toward the end of the path $\Gamma(\Pi)$. Now, since at most $\lceil \|\mathrm{y}_\ell - \mathrm{y}_{\ell-1}\|_2 / \delta_{\min} \rceil$ polytopes are required per line segment, summing over all segments forming the path $\Gamma(\Pi)$ yields 
\begin{equation}\label{eq:J_max}
J \leq J_\mathrm{max} := \sum_{\ell=1}^{L} \left\lceil \frac{\|\mathrm{y}_\ell - \mathrm{y}_{\ell-1}\|_2}{\delta_{\min}} \right\rceil < \infty\enspace.
\end{equation}
This implies that the \texttt{PathCover} algorithm terminates in finitely many steps, producing a sequence of polytopes $\mathcal{G} = \{\mathcal{H}_1, \ldots, \mathcal{H}_J\}$ satisfying conditions~\ref{PCC:1}--\ref{PCC:3}.
\end{proof}

\section{Proof of the RISP Complexity Theorem}
\label{app:theorem2}

\begin{proof}[Proof of Theorem~\ref{thm:RISP_complexity}]
Fix $\mathrm{y}_\mathrm{seed} \in \operatorname{int}(\mathcal{F})$.
We begin with the expected time complexity.
Since we process $N_m$ points at each iteration, the per-iteration cost is proportional to $N_m$ and the total cost is
\begin{equation*}
    T(n) = \sum_{m=0}^{\tau-1} N_m
\end{equation*}
where $\tau = \inf\{m \geq 0 \mid N_m = 0\}$ is the maximum number of iterations until exhaustion of all points in $\mathcal{P}$. To find a bound on the expected time complexity $\mathbb{E}[T(n)]$, define the event $\mathcal{A}_m = \{\, N_{m+1} \leq (1-\beta)\,N_m\}$ that iteration $m$ eliminates at least a fraction $\beta$ of the remaining points, and let $p_m = \Pr(\mathcal{A}_m \mid \mathcal{P}_m)$ be the corresponding probability conditioned on the current point cloud $\mathcal{P}_m$; this event captures both the geometric structure and the cardinality of the point cloud.
If $\mathcal{A}_m^c$ denotes the complement of the event $\mathcal{A}_m$, the law of total expectation implies
\begin{align}
    \nonumber
    \mathbb{E}[N_{m+1} \mid \mathcal{P}_m] &= p_m \,\mathbb{E}[N_{m+1} \mid \mathcal{A}_m, \mathcal{P}_m] \\
    \label{eq:total_expect}
    &+ (1 - p_m) \, \mathbb{E}[N_{m+1} \mid \mathcal{A}_m^c, \mathcal{P}_m] \enspace .
\end{align}

For the first term in \eqref{eq:total_expect}, we use the fact that $N_{m+1} \leq (1-\beta) N_m$ on $\mathcal{A}_m$ and the monotonicity of expectation to get
\begin{equation}\label{eq:bound_1}
    \mathbb{E}[N_{m+1} \mid \mathcal{A}_m, \mathcal{P}_m] \leq \mathbb{E}[(1-\beta)N_{m} \mid \mathcal{A}_m, \mathcal{P}_m]
= (1-\beta)N_{m}
\end{equation}
where the last equality follows from $\mathbb{E}[N_{m} \mid \mathcal{A}_m, \mathcal{P}_m] = N_m$ since $|\mathcal{P}_m|=N_m$.
For the second term in \eqref{eq:total_expect}, observe that $N_{m+1} \leq N_m - 1$ deterministically, since at every iteration the point $\mathrm{p}_m$ itself is always removed. As a result,
\begin{align}
    \nonumber
    \mathbb{E}[N_{m+1} \mid \mathcal{A}_m^c,\, \mathcal{P}_m] &\leq \mathbb{E}[N_{m} - 1 \mid \mathcal{A}_m^c, \, \mathcal{P}_m]
    \\
    \label{eq:bound_2}
    &\leq \mathbb{E}[N_{m} \mid \mathcal{A}_m^c, \, \mathcal{P}_m] = N_{m}
\end{align}
where the second inequality follows from monotonicity of the expectation and the last equality follows from $\mathbb{E}[N_{m} \mid \mathcal{A}_m^c, \mathcal{P}_m] = N_m$ since $|\mathcal{P}_m|=N_m$.
Substituting the bounds \eqref{eq:bound_1} and \eqref{eq:bound_2} in \eqref{eq:total_expect} results in
\begin{align}
\nonumber
\mathbb{E}[N_{m+1} \mid \mathcal{P}_m]
&\leq p_m(1-\beta)N_m + (1 - p_m)N_m
\\
\label{eq:recursive_bound}
&= (1 - p_m\beta)N_m
\leq (1 - r\beta)N_m
\end{align}
where the last inequality follows from~\eqref{eq:beta-good}.

Now, given an instance $\mathcal{P}$ for the initial point cloud $\mathcal{P}_0$, i.e.\ $\mathcal{P}_0 = \mathcal{P}$ and $|\mathcal{P}| =|\mathcal{P}_0| = N_0 = n$, using \eqref{eq:recursive_bound} and applying induction on $m$ yields
\begin{equation}\label{eq:geometric_decay}
    \mathbb{E}[N_m] \leq (1-r\beta)^m n \enspace .
\end{equation}
Then, noting that $N_m = 0$ deterministically for all $m \geq \tau$,
\begin{equation*}
  \mathbb{E}[T(n)]
  = \sum_{m=0}^{\infty}\mathbb{E}[N_m]
  \leq n\sum_{m=0}^{\infty}(1-r\beta)^m
  = \frac{n}{r\beta} \enspace ,
\end{equation*}
which, since $r>0$ and $\beta>0$, implies that $\mathbb{E}[T(n)] = O(n)$.

It remains to account for redundant-constraint removal (Algorithm~\ref{alg:risp}, line 14), whose cost is dominated by the \texttt{Quickhull} algorithm operating on the $\tau$ accumulated dual points at expected cost $O(\tau \log \tau)$~\cite{quickhull}. It thus suffices to bound $\mathbb{E}[\tau]$, which we do via the multiplicative drift theorem~\cite[Theorem~7]{Doerr2010drift}; we verify its hypotheses below.
By Algorithm~\ref{alg:risp}, $\mathcal{P}_{m+1}$ depends only on $\mathcal{P}_m$ and the uniform sample $\mathrm{p}_m\in\mathcal{P}_m$, so $\{\mathcal{P}_m\}_{m \geq 0}$ is a Markov process; since points are removed at every iteration, $\mathcal{P}_m \subseteq \mathcal{P}_0$ for all $m$, and the state space lies in the \emph{finite} power set $2^{\mathcal{P}_0}$ (of cardinality $2^n<\infty$). The cardinality map $|\cdot|: 2^{\mathcal{P}_0} \to \mathbb{Z}_{\geq 0}$ is the integer-valued, bounded potential required by~\cite[Theorem~7]{Doerr2010drift}, with first hitting time of zero $\tau = \inf\{m\geq 0 \mid N_m = 0\}$.
It remains to verify the multiplicative drift condition for $N_m$. Let $l \geq 1$ be an integer with $\Pr\{N_m = l\} > 0$; since $l \geq 1$, the event $\{N_m = l\}$ lies in $\{m < \tau\}$, where~\eqref{eq:recursive_bound} holds. Evaluating $\mathbb{E}[N_{m+1} \mid N_m]$ using the law of iterated expectations yields
\begin{align*}
\nonumber
\mathbb{E}[N_{m+1} \mid N_m]
&= \mathbb{E}\big[\,\mathbb{E}[N_{m+1} \mid \mathcal{P}_m] \,\big|\, N_m\,\big] \\
&\leq \mathbb{E}\big[(1 - r\beta)\,N_m \,\big|\, N_m\big] = (1 - r\beta)\,N_m 
\end{align*}
where the bound~\eqref{eq:recursive_bound} was used. Now conditioning the difference $N_m - N_{m+1}$ on the event $N_m = l$, we obtain
\begin{equation*}
    \mathbb{E}[N_m - N_{m+1} \mid N_m = l]
    = l - \mathbb{E}[N_{m+1} \mid N_m = l]
    \geq r\beta\,l
\end{equation*}
for all integers $l \geq 1$, which is the multiplicative drift condition of~\cite[Theorem~7]{Doerr2010drift} with rate $r\beta>0$. Hence, applying the theorem yields
\begin{equation*}
\mathbb{E}[\tau \mid \mathcal{P}_0] \leq \frac{1 + \ln(|\mathcal{P}_0|)}{r\beta} = \frac{1 + \ln n}{r\beta} = O(\log n) \enspace .
\end{equation*}
Combining with the deterministic bound $\tau \leq n$ gives
\begin{equation*}
  \mathbb{E}[\tau \log \tau \mid \mathcal{P}_0] \;\leq\; (\log n)\,\mathbb{E}[\tau \mid \mathcal{P}_0] \;=\; O(\log^2 n) \enspace .
\end{equation*}
Since $O(\log^2 n)$ is dominated by the $O(n)$ term, the overall expected time complexity of \texttt{RISP} is $O(n)$.

Finally, we turn our attention to the worst-case complexity. If at every iteration the sampled point $\mathrm{p}_m$ is the only point removed, then $\tau = n$ and
\begin{equation}\nonumber
  T(n) = n + (n-1) + \cdots + 1 = \frac{n(n+1)}{2} = O(n^2) \enspace .
\end{equation}
In this case, \texttt{Quickhull} operates on $\tau = n$ dual points and incurs a worst-case cost $O(n^2)$, which does not affect the overall $O(n^2)$ bound, completing the proof.
\end{proof}

\section{Proof of the PathCover Complexity Theorem}
\label{app:theorem3}

Let $m_j$ denote the number of halfspaces defining the $j$-th polytope $\mathcal{H}_j$, $1 \leq j \leq J$. Since each halfspace generated by \texttt{RISP} eliminates at least one point from the input point cloud, each polytope contains at most $n$ generated halfspaces; hence $m_j \leq n$ and $\sum_{j=1}^{J} m_j \leq Jn.$

\emph{RISP calls:} The algorithm invokes \texttt{RISP} once for each generated polytope. By Theorem~\ref{thm:RISP_complexity}, each call has expected cost $O(n)$ and worst-case cost $O(n^2)$. Therefore, the total cost of all \texttt{RISP} calls is $O(Jn)$ in expectation and $O(Jn^2)$ in the worst case.

\emph{Waypoint membership tests:} Each membership test in $\mathcal{H}_j$ requires evaluating (Algorithm~\ref{alg:PathCover}, line 6) its $m_j$ halfspace inequalities and therefore costs $O(m_j) \leq O(n)$. The main loop performs at most one membership test per iteration. There are at most $L+1$ iterations in which the waypoint index $\ell$ is advanced, and at most $J-1$ iterations in which a new polytope is generated. Thus, the total number of membership tests is at most $L+J$, giving a total cost of $O((L+J)n)$.

\emph{Intersection queries:} A line-polytope intersection query with $\mathcal{H}_j$ also costs $O(m_j)$. Such a query is performed only when a new polytope is generated, so the total intersection cost is bounded by $\sum_{j=1}^{J} O(m_j) \leq O(Jn)$. 

Finally, combining these bounds and using the fact $J \leq J_\mathrm{max}$ by \eqref{eq:J_max}, the expected complexity can be computed as
\begin{align*}
\mathbb{E}[T(n,J,L)]
&\leq O(Jn) + O((L+J)n) + O(Jn) \\
&\leq O((3J_\mathrm{max}+L)n) = O(n) \enspace .
\end{align*}
Similarly, the worst-case running time satisfies
\begin{align*}
T(n,J,L)
&= O(Jn^2) + O((L+J)n) + O(Jn) \\
&\leq O(J_\mathrm{max}n^2+Ln + 2J_\mathrm{max}n) = O(n^2)
\end{align*}
since the $O(n)$ terms are dominated by $O(n^2)$. 

\bibliographystyle{IEEEtran}
\phantomsection
\bibliography{IEEEabrv, References}

\begin{thebibliography}{10}
\providecommand{\url}[1]{#1}
\csname url@samestyle\endcsname
\providecommand{\newblock}{\relax}
\providecommand{\bibinfo}[2]{#2}
\providecommand{\BIBentrySTDinterwordspacing}{\spaceskip=0pt\relax}
\providecommand{\BIBentryALTinterwordstretchfactor}{4}
\providecommand{\BIBentryALTinterwordspacing}{\spaceskip=\fontdimen2\font plus
\BIBentryALTinterwordstretchfactor\fontdimen3\font minus \fontdimen4\font\relax}
\providecommand{\BIBforeignlanguage}[2]{{%
\expandafter\ifx\csname l@#1\endcsname\relax
\typeout{** WARNING: IEEEtran.bst: No hyphenation pattern has been}%
\typeout{** loaded for the language `#1'. Using the pattern for}%
\typeout{** the default language instead.}%
\else
\language=\csname l@#1\endcsname
\fi
#2}}
\providecommand{\BIBdecl}{\relax}
\BIBdecl

\bibitem{How2009TCST}
Y.~Kuwata, J.~Teo, G.~Fiore, S.~Karaman, E.~Frazzoli, and J.~P. How, ``Real-time motion planning with applications to autonomous urban driving,'' \emph{IEEE Trans. Control Syst. Technol.}, vol.~17, no.~5, pp. 1105--1118, 2009.

\bibitem{Havoutis2020RAL}
M.~Brandão, O.~B. Aladag, and I.~Havoutis, ``{GaitMesh}: Controller-aware navigation meshes for long-range legged locomotion planning in multi-layered environments,'' \emph{IEEE Robot. Automat. Lett.}, vol.~5, no.~2, pp. 3596--3603, 2020.

\bibitem{Kim2018RAL}
H.~Lee, H.~Kim, W.~Kim, and H.~J. Kim, ``An integrated framework for cooperative aerial manipulators in unknown environments,'' \emph{IEEE Robot. Automat. Lett.}, vol.~3, no.~3, pp. 2307--2314, 2018.

\bibitem{Gao2025Universal}
M.~Zhang, N.~Chen, H.~Wang, J.~Qiu, Z.~Han, Q.~Ren, C.~Xu, F.~Gao, and Y.~Cao, ``Universal trajectory optimization framework for differential drive robot class,'' \emph{IEEE Trans. Autom. Sci. Eng.}, vol.~22, pp. 13\,030--13\,045, 2025.

\bibitem{Park2023TRO}
J.~Park, Y.~Lee, I.~Jang, and H.~J. Kim, ``{DLSC}: Distributed multi-agent trajectory planning in maze-like dynamic environments using linear safe corridor,'' \emph{IEEE Trans. Robot.}, vol.~39, no.~5, pp. 3739--3758, 2023.

\bibitem{Mora2023RAL}
G.~Chen, S.~Wu, M.~Shi, W.~Dong, H.~Zhu, and J.~Alonso-Mora, ``{RAST}: Risk-aware spatio-temporal safety corridors for mav navigation in dynamic uncertain environments,'' \emph{IEEE Robot. Automat. Lett.}, vol.~8, no.~2, pp. 808--815, 2023.

\bibitem{How2022FASTER}
J.~Tordesillas, B.~T. Lopez, M.~Everett, and J.~P. How, ``{FASTER}: Fast and safe trajectory planner for navigation in unknown environments,'' \emph{IEEE Trans. Robot.}, vol.~38, no.~2, pp. 922--938, 2022.

\bibitem{Narkh2022RAL}
K.~S. Narkhede, A.~M. Kulkarni, D.~A. Thanki, and I.~Poulakakis, ``A sequential {MPC} approach to reactive planning for bipedal robots using safe corridors in highly cluttered environments,'' \emph{IEEE Robot. Automat. Lett.}, vol.~7, no.~4, pp. 11\,831--11\,838, 2022.

\bibitem{Gao2025TRO}
Q.~Wang, Z.~Wang, M.~Wang, J.~Ji, Z.~Han, T.~Wu, R.~Jin, Y.~Gao, C.~Xu, and F.~Gao, ``Fast iterative region inflation for computing large {2-D/3-D} convex regions of obstacle-free space,'' \emph{IEEE Trans. Robot.}, vol.~41, pp. 3223--3243, 2025.

\bibitem{Deits2015mixedInteger}
R.~Deits and R.~Tedrake, ``Efficient mixed-integer planning for {UAV}s in cluttered environments,'' in \emph{Proc. IEEE Int. Conf. Robot. Autom.}, 2015, pp. 42--49.

\bibitem{Liu2017RAL}
S.~Liu, M.~Watterson, K.~Mohta, K.~Sun, S.~Bhattacharya, C.~J. Taylor, and V.~Kumar, ``Planning dynamically feasible trajectories for quadrotors using safe flight corridors in {3-D} complex environments,'' \emph{IEEE Robot. Automat. Lett.}, vol.~2, no.~3, pp. 1688--1695, 2017.

\bibitem{Park2022RAL}
J.~Park, D.~Kim, G.~C. Kim, D.~Oh, and H.~J. Kim, ``Online distributed trajectory planning for quadrotor swarm with feasibility guarantee using linear safe corridor,'' \emph{IEEE Robot. Automat. Lett.}, vol.~7, no.~2, pp. 4869--4876, 2022.

\bibitem{Gao2022bubble}
Y.~Ren, F.~Zhu, W.~Liu, Z.~Wang, Y.~Lin, F.~Gao, and F.~Zhang, ``Bubble planner: Planning high-speed smooth quadrotor trajectories using receding corridors,'' in \emph{Proc. IEEE/RSJ Int. Conf. Intell. Robots Syst.}, 2022, pp. 6332--6339.

\bibitem{Deits2015IRIS}
R.~Deits and R.~Tedrake, ``Computing large convex regions of obstacle-free space through semidefinite programming,'' in \emph{Algorithmic Foundations of Robotics XI}, H.~L. Akin, N.~M. Amato, V.~Isler, and A.~F. van~der Stappen, Eds.\hskip 1em plus 0.5em minus 0.4em\relax Springer, 2015, pp. 109--124.

\bibitem{Tedrake2023IRISNLP}
M.~Petersen and R.~Tedrake, ``Growing convex collision-free regions in configuration space using nonlinear programming,'' \emph{arXiv:2303.14737}, 2023.

\bibitem{Tedrake2024IJRR}
H.~Dai, A.~Amice, P.~Werner, A.~Zhang, and R.~Tedrake, ``Certified polyhedral decompositions of collision-free configuration space,'' \emph{Int. J. Robot. Res.}, vol.~43, no.~9, pp. 1322--1341, 2024.

\bibitem{Tedrake2024IRISZO}
P.~Werner, T.~Cohn, R.~H. Jiang, T.~Seyde, M.~Simchowitz, R.~Tedrake, and D.~Rus, ``Faster algorithms for growing collision-free convex polytopes in robot configuration space,'' \emph{Int. J. Robot. Res}, vol.~0, no.~0, p. 02783649261436917, 2026.

\bibitem{ren2025CIRI}
Y.~Ren, F.~Zhu, G.~Lu, Y.~Cai, L.~Yin, F.~Kong, J.~Lin, N.~Chen, and F.~Zhang, ``Safety-assured high-speed navigation for mavs,'' \emph{Sci. Robot.}, vol.~10, no.~98, p. eado6187, 2025.

\bibitem{Kumar2025RAL}
Y.~Wu, I.~Spasojevic, P.~Chaudhari, and V.~Kumar, ``Towards optimizing a convex cover of collision-free space for trajectory generation,'' \emph{IEEE Robot. Automat. Lett.}, vol.~10, no.~5, pp. 4762--4769, 2025.

\bibitem{Sergei2017stereo}
S.~Savin, ``An algorithm for generating convex obstacle-free regions based on stereographic projection,'' in \emph{Proc. Int. Siberian Conf. Control Commun.}, 2017, pp. 1--6.

\bibitem{Gao2020pointcloud}
X.~Zhong, Y.~Wu, D.~Wang, Q.~Wang, C.~Xu, and F.~Gao, ``Generating large convex polytopes directly on point clouds,'' \emph{arXiv:2010.08744}, 2020.

\bibitem{Gao2022ICRA}
T.~Liu, Q.~Wang, X.~Zhong, Z.~Wang, C.~Xu, F.~Zhang, and F.~Gao, ``Star-convex constrained optimization for visibility planning with application to aerial inspection,'' in \emph{Proc. IEEE Int. Conf. Robot. Autom.}, 2022, pp. 7861--7867.

\bibitem{quickhull}
C.~B. Barber, D.~P. Dobkin, and H.~Huhdanpaa, ``The quickhull algorithm for convex hulls,'' \emph{ACM Trans. Math. Softw.}, vol.~22, no.~4, p. 469–483, 1996.

\bibitem{Kumar2018JMR}
M.~Kennedy, III, D.~Thakur, M.~Ani~Hsieh, S.~Bhattacharya, and V.~Kumar, ``Optimal paths for polygonal robots in {SE(2)},'' \emph{J. Mechanisms Robot.}, vol.~10, no.~2, p. 021005, 2018.

\bibitem{Gao2020TeachRepeat}
F.~Gao, L.~Wang, B.~Zhou, X.~Zhou, J.~Pan, and S.~Shen, ``{Teach-Repeat-Replan}: A complete and robust system for aggressive flight in complex environments,'' \emph{IEEE Trans. Robot.}, vol.~36, no.~5, pp. 1526--1545, 2020.

\bibitem{Lambert2022VoxelGrid}
C.~Toumieh and A.~Lambert, ``Voxel-grid based convex decomposition of {3D} space for safe corridor generation,'' \emph{J. Intell. Robot. Syst.}, vol. 105, no.~4, p.~87, 2022.

\bibitem{hornung13auro}
A.~Hornung, K.~M. Wurm, M.~Bennewitz, C.~Stachniss, and W.~Burgard, ``{OctoMap}: an efficient probabilistic {3D} mapping framework based on octrees,'' \emph{Auton. Robots}, vol.~34, no.~3, pp. 189--206, 2013.

\bibitem{Gao2018ICRA}
F.~Gao, W.~Wu, Y.~Lin, and S.~Shen, ``Online safe trajectory generation for quadrotors using fast marching method and bernstein basis polynomial,'' in \emph{Proc. IEEE Int. Conf. Robot. Autom.}, 2018, pp. 344--351.

\bibitem{Lambert2021GPU}
C.~Toumieh and A.~Lambert, ``{GPU} accelerated voxel grid generation for fast {MAV} exploration,'' \emph{arXiv:2112.13169}, 2021.

\bibitem{LaValle2006Planning}
S.~M. LaValle, \emph{Planning Algorithms}.\hskip 1em plus 0.5em minus 0.4em\relax Cambridge, UK: Cambridge University Press, 2006.

\bibitem{Harabor2011jps}
D.~Harabor and A.~Grastien, ``Online graph pruning for pathfinding on grid maps,'' in \emph{Proc. AAAI Conf. Artif. Intell.}\hskip 1em plus 0.5em minus 0.4em\relax AAAI Press, 2011, p. 1114–1119.

\bibitem{Gao2022Gcopter}
Z.~Wang, X.~Zhou, C.~Xu, and F.~Gao, ``Geometrically constrained trajectory optimization for multicopters,'' \emph{IEEE Trans. Robot.}, vol.~38, no.~5, pp. 3259--3278, 2022.

\bibitem{Mellinger2011MinimumSnap}
D.~Mellinger and V.~Kumar, ``Minimum snap trajectory generation and control for quadrotors,'' in \emph{Proc. IEEE Int. Conf. Robot. Autom.}, 2011, pp. 2520--2525.

\bibitem{Lee2010CDC}
T.~Lee, M.~Leok, and N.~H. McClamroch, ``Geometric tracking control of a quadrotor uav on {SE(3)},'' in \emph{Proc. IEEE Conf. Decis. Control}, 2010, pp. 5420--5425.

\bibitem{fatrop}
L.~Vanroye, A.~Sathya, J.~De~Schutter, and W.~Decré, ``{FATROP}: A fast constrained optimal control problem solver for robot trajectory optimization and control,'' in \emph{Proc. IEEE/RSJ Int. Conf. Intell. Robots Syst.}, 2023, pp. 10\,036--10\,043.

\bibitem{Boyd_cvx}
S.~Boyd and L.~Vandenberghe, \emph{Convex Optimization}.\hskip 1em plus 0.5em minus 0.4em\relax Cambridge, UK: Cambridge University Press, 2004.

\bibitem{Doerr2010drift}
B.~Doerr, D.~Johannsen, and C.~Winzen, ``Drift analysis and linear functions revisited,'' in \emph{Proc. IEEE Congr. Evol. Comput.}, 2010, pp. 1--8.

\end{thebibliography}

\end{document}